\documentclass{article}

\usepackage[letterpaper,top=2cm,bottom=2cm,left=3cm,right=3cm,marginparwidth=1.75cm]{geometry}

\usepackage{xcolor}
\usepackage{float}
\usepackage{amssymb}
\usepackage{amstext}
\usepackage{amsmath}
\usepackage{amscd}
\usepackage{amsthm}
\usepackage{amsfonts}
\usepackage{enumerate}
\usepackage{graphicx}
\usepackage{latexsym}
\usepackage{mathrsfs}
\usepackage{mathtools}
\usepackage{cases}
\usepackage{verbatim}
\usepackage{bm}
\usepackage{tikz}
\usepackage{caption}
\usepackage{tikz}
\usetikzlibrary{matrix}
\usepackage{hyperref} 
\usepackage{comment}
\usepackage{tikz-cd}
\usepackage{pb-diagram}
\usepackage{cleveref}

\newcommand{\RR}{\mathbb{R}}

\newcommand{\NN}{\mathbb{N}}
\newcommand{\Pp}{\mathcal{P}}

\newcommand{\dotp}[2]{\langle #1,\,#2\rangle}
\newcommand{\supp}{\hbox{supp}\,}

\newtheorem{lemma}{Lemma}
\newtheorem{theorem}{Theorem}
\newtheorem{corollary}{Corollary}
\newtheorem{proposition}{Proposition}
\newtheorem{definition}{Definition}

\newtheorem{remark}{Remark}

\renewcommand{\d}[1]{\mathrm{d}}

\renewcommand*\d{\mathop{}\!\mathrm{d}}

\def \p {\partial }
\def \e {\epsilon }
\def \R {{\mathbb R}}

\def \N {{\mathbb N}}

\def \ba {\begin {eqnarray*} }
\def \ea {\end {eqnarray*} }
\def \beq {\begin {eqnarray}}
\def \eeq {\end {eqnarray}}

\newcommand{\bra}{\langle}

\newcommand{\ket}{\rangle}

\newcommand{\commented}[1]{}

\title{{\bf Transformers through the lens of support-preserving maps between measures}}
\author{%
  Takashi Furuya$^{1}$
  \quad
  Maarten V. de Hoop$^{2}$
  \quad
  Matti Lassas$^{3}$
  \vspace{5mm}
  \\
  $^{1}$ Doshisha University, RIKEN AIP, \texttt{takashi.furuya0101@gmail.com}\\
  $^{2}$Rice University, 
  \texttt{mdehoop@rice.edu}\\
  $^{3}$University of Helsinki, \texttt{matti.lassas@helsinki.fi}
}

\date{}

\begin{document}
\maketitle

\begin{abstract}
Transformers are deep architectures that define ``in-context maps'' which enable predicting new tokens based on a given set of tokens (such as a prompt in NLP applications or a set of patches for a vision transformer). In previous work, we studied the ability of these architectures to handle an arbitrarily large number of context tokens. To mathematically, uniformly analyze their expressivity, we considered the case that the mappings are conditioned on a context represented by a probability distribution which becomes discrete for a finite number of tokens. Modeling neural networks as maps on probability measures has multiple applications, such as studying Wasserstein regularity, proving generalization bounds and doing a mean-field limit analysis of the dynamics of interacting particles as they go through the network. In this work, we study the question what kind of maps between measures are transformers. We fully characterize the properties of maps between measures that enable these to be represented in terms of in-context maps via a push forward. On the one hand, these include transformers; on the other hand, transformers universally approximate representations with any continuous in-context map. These properties are preserving the cardinality of support and that the regular part of their Fr\'{e}chet derivative is uniformly continuous. Moreover, we show that the solution map of the Vlasov equation, which is of nonlocal transport type, for interacting particle systems in the mean-field regime for the Cauchy problem satisfies the conditions on the one hand and, hence, can be approximated by a transformer; on the other hand, we prove that the measure-theoretic self-attention has the properties that ensure that the infinite depth, mean-field measure-theoretic transformer can be identified with a Vlasov flow. 
\end{abstract}

\section{Introduction}

Transformers have revolutionized the field of machine learning with their powerful attention mechanisms as introduced by \cite{vaswani2017attention}. The exceptional performance and expressivity of large-scale transformers have been empirically well established for both NLP (\cite{brown2020language}) and vision applications (\cite{dosovitskiy2020image}). One key property of these architectures is their ability to leverage contexts of arbitrary length, which enables the parameterization of ``in-context'' mappings with an arbitrarily large complexity. The previous work (\cite{furuya2024transformers}) studied this by analyzing the expressivity of mappings that are conditioned on a context represented by a probability distribution of tokens which becomes discrete for a finite number of these. By implication, transformers are viewed as maps between measures. Here, we present a full characterization of maps between measures that can be represented by these measure-theoretic transformers, that is, we address the question which class of mappings between measures can be identified with transformers.

\paragraph{Mathematical modeling of transformers.} It is now customary to describe transformers as performing ``in context'' prediction, which means that it maps token to token, while this map depends on a set of previously seen tokens. The size of this context might be very long, possibly arbitrarily long, which has been addressed in \cite{furuya2024transformers} that concerns the transformers as universal in-context learners. The ability of trained transformers to effectively perform in-context computation has been supported by both empirical studies (\cite{von2023uncovering}) and theoretical ones (\cite{ahn2024transformers,mahankali2023one,sander2024transformers,zhang2023trained}) on simplified architectures (typically with linear attention) and specific data generation processes. The connection between transformers and graph neural networks is exposed in~\cite{muller2023attending}.

It has been noted that in order to make a comprehensive analysis of arbitrarily long token lengths, and to describe a ``mean-field'' limit of an infinite number of tokens, it is natural to view attention as operating over probability distributions of tokens (\cite{vuckovic2020mathematical,sander2022sinkformers}). The regularity (Lipschitz continuity) of the resulting attention layers was analyzed in \cite{castin2024smooth}.

Deep transformers (with residual or skip connections) have been described by a coupled system of particles evolving across the layers. Such systems are fundamental in modeling phenomena across physics, biology, and engineering. This connection has been exploited by \cite{geshkovski2024} who studied measure-to-measure interpolation using transformers. The analysis of the clustering properties of such an evolution was studied in \cite{geshkovski2023emergence,geshkovski2023mathematical}. 
\cite{biswal2024universal} further investigate the use of transformers to approximate the mean-field dynamics of interacting particle systems exhibiting collective behavior. They establish theoretical bounds on the distance between the true mean-field dynamics and those obtained using a transformer, by lifting it from a sequence-to-sequence map to a map on measures upon taking the expectation of a finite-dimensional transformer with respect to a product measure. From a different viewpoint, this connection will be further developed here, in general for mappings between measures satisfying the conditions to be representable by measure-theoretic transformers. The structure of the interacting particle system enables concrete connections to established mathematical topics, including nonlinear, nonlocal transport equations, Wasserstein gradient flows, and collective behavior models.

\paragraph{Universality of transformers.} \cite{yun2019transformers} provides, to the best of our knowledge, the most detailed account of the universality of transformers. The authors rely on shallow transformers with only two heads and require that the transformers operate over an embedding dimension which grows with the number of tokens. This result is refined in~\cite{nath2024transformers} and emphasizes the difficulty of attention mechanisms to capture smooth functions. 

We note that there exist variations of the original transformer's architecture which enjoy universality results, for instance, the Sumformer (\cite{alberti2023sumformer}) and stochastic deep network (\cite{de2019stochastic}); these also require an embedding dimension that grows with the number of tokens. We furthermore mention the introduction of probabilistic transformers (\cite{kratsios2023small}) which can approximate embeddings of metric spaces. The work of~\cite{agrachev2024generic} provides an abstract universal interpolation result for equivariant architectures under genericity conditions; however, it is not known whether there exist generic attention maps. 

While this is not directly related to the analysis presented here, some works study the expressivity of transformers when operating on a discrete set of tokens as formal systems (\cite{chiang2023tighter,merrill2023expresssive,strobl2024formal,elhage2021mathematical}). Another line of work studies the impact of positional encoding on their expressivity (\cite{luo2022your}).

\cite{furuya2024transformers} provide a rigorous formalization of transformer expressivity and continuity as operating over the space of probability distributions through its in-context mapping. The main mathematical result is the universal approximation of in-context mappings for the unmasked and the masked settings, considering deep transformers with a fixed embedding dimension, but which are universal for an arbitrary number of tokens. A more constructive approach, although applicable to a narrower class of functions, is proposed by \cite{wange2024}. \cite{sander2025} introduce a framework to analyze the expressivity of deep transformers in next-token prediction, while exploring how successive attention layers solve a causal kernel least squares regression problem to predict the next token accurately.

\subsection{Our contributions}

The central question posed here, is whether a support-preserving map between measures can be characterized as the push forward with an in-context map or not. We answer this question in the affirmative by introducing a ``certain'' smoothness condition, which roughly entails that a ``certain'' derivative of the map is uniformly continuous. We provide a counterexample, showing that this condition is essential. Our proof is essentially constructive.

Applying this result and the underlying analysis, we prove that measure-theoretic transformers approximate such support-preserving maps, using the results of \cite{furuya2024transformers}. This settles the full characterization of measure-theoretic transformers.

Finally, we show that the solution operator of the Vlasov equation, which is of nonlocal transport type, for the Cauchy problem satisfies the above mentioned condition(s) as a map between initial and final measures. This provides a bridge between interacting particle systems, in the mean-field regime, in the general context of measure-theoretic transformers. (A second-order generalization of measure-theoretic transformers yields a similar result for the solution operator of the kinetic Cucker-Smale equation (\cite{biswal2024universal}).)

We first present the analysis relating support-preserving maps between measures with in-context maps that define measure-theoretic transformers.
We later show, in an appendix, that ``classical'' transformers arise as a limiting case through (sub)sequences of discrete measures determined by tokens. The correspondence with Vlasov flows is established in the mean-field sense and is based on an infinite-depth limit.

% \noindent {\color{red} Based on our analysis, discuss why the choice of softmax  (smoothness) is natural in transformers. Discuss the paper by V. Castin et al. on How smooth is attention? \\ https://arxiv.org/abs/2312.14820, in particular the result on the Lipschitz constant. How does this relate to the Lipschitz constant of the in-context mappings, $G$?}

\subsection{Notation}

Let $\Omega \subset \RR^d$ be a compact set. We denote by $\mathcal{P}(\Omega)$ the space of probability measures on $\Omega$.
Below, all measures $\mu$ on subset $\Omega$ of $\RR^d$ are defined 
on the $\sigma$-algebra of the Borel sets of $\Omega$. We denote by $C(\Omega)$ the space of continuous functions from $\Omega$ to $\R$,
and the dual coupling between $\varphi \in C(\Omega)$ and $\mu \in \Pp(\Omega)$ by 
$$
\dotp{\varphi}{\mu}
:= \int_\Omega \varphi(x) \mathrm{d} \mu(x).
$$
We use the notations of Wasserstein distance as $W_p$ for $1 \leq p < \infty$. We extend $\mathcal{P}(\Omega)$, that is, the set of all probability measures to the set of all strictly positive, finite measures 
$$ 
\mathcal M^+(\Omega)=\{s \mu :\ \mu \in \Pp(\Omega), \ s>0  \}.
$$
We also extend the $W_1$ distance to $\mathcal M^+(\Omega)$  by defining for  $\mu_1,\mu_2 \in \Pp(\Omega) $ and $s_1,s_2>0$ 
$$
 W_1(s_1\mu_1,s_2\mu_2)= W_1(\mu_1,\mu_2)+|s_1-s_2|,
$$
see \cite{extensionWasserstein}.
We write
$$
\mathcal M^+_{fin,(n)}(\Omega) := \left\{\sum_{i=1}^n a_i\delta_{x_i} \in \mathcal M^+(\Omega):\ x_i \in X,\ a_i > 0
 \right\},
$$
$$
\mathcal M^+_{fin}(\Omega) :=\bigcup_{n=1}^\infty \mathcal M^+_{fin,(n)}(\Omega) = \left\{\sum_{i=1}^n a_i\delta_{x_i} \in \mathcal M^+(\Omega):\ x_i \in \Omega,\ a_i > 0,\ n\in \NN \right\}.
$$
Finally, we denote by $\mathcal M^+_{fin,dif,(n)}(\Omega)$ the measures of the form 
$
    \mu = \sum_{i=1}^{n} a_i \delta_{x_i} \in \mathcal M_{fin,(n)}(\Omega),
$   
 where $a_j> 0$ and for all non-empty subsets $J,K \subset \{1,2,\dots,n\}$ satisfying $J\cap K=\emptyset$ it holds that
$$
\sum_{j \in J} a_j \not= \sum_{k \in K} a_k.
$$
We set $\mathcal M^+_{fin,dif}(\Omega)=\bigcup_{n=1}^\infty \mathcal M^+_{fin,dif,(n)}(\Omega)$. 
For a continuous map $g:\Omega\to \Omega$ and a measure $\mu$ the push-forward measure of $\mu$ in the map $g$
is the measure $g_\#\mu(A):=\mu(g^{-1}(A))$, where $A\subset \Omega$ is an open set. For further details pertaining to these notions, we refer to Appendix~\ref{sec:A.1}.

We state the following lemma, which is proved in Appendix~\ref{app:dense}.
\begin{lemma}\label{lem:dense}
$\mathcal M^+_{fin,dif}(\Omega)$ is dense in $\mathcal M^+(\Omega)$ in the $1$-Wasserstein topology.
\end{lemma}

\section{Definitions and properties of the relevant maps}

\subsection{Support-preserving maps and in-context maps}
\label{subsec: in context maps}

\begin{definition}\label{def:support-preserving-maps}  
%Let $\Omega\subset \RR^d$. 
We say that $f: \mathcal M^+(\Omega) \to \mathcal M^+(\RR^{d'})$ is a support-preserving map if for all
finitely supported measures of the form,
\beq\label{finite sum 1}
\mu = \sum_{i=1}^{n} a_i \delta_{x_i} \in \mathcal M^+_{fin}(\Omega),
\eeq
where $a_i=\frac {c}n$, $c \in \R_+,$ and $x_i\in \Omega$,  there exist $y_1,...,y_n \in \RR^{d'}$ such that 
\beq\label{finite sum 2}
f(\mu)=\sum_{i=1}^{n} a_i \delta_{y_i} \in \mathcal M^+_{fin}(\RR^{d'})
\eeq
and satisfy the condition
\beq\label{preserves tokens}
\hbox{if $x_j=x_i$ then $y_j=y_i$.}
\eeq
\end{definition}

%\Takashi{In the proof (Section 3.2), we used the different coefficients $a_i$, i.e., we considered $\mu \in \mathcal M^+_{fin,dif,(n)}(\Omega)$. Can we replace all proof when $\mu$ is uniform measure ? I am not sure, for example, in the continuous limit, we need Lemma 1. }

%{\color{cyan} 
The consideration of support-preserving maps to study transformers is natural; see Section~\ref{ssub:4.1} and formula \eqref{map F Gamma} in Appendix \ref{sec:A.3} and \ref{sec:A.4} for a detailed discussion. Let $(x_1,x_2,\dots,x_n) \in \Omega^n$ be the sequences of $n$ tokens in $\Omega\subset {\RR^d}$, and let the union of all these be $X_d = \bigcup_{n=1}^\infty \Omega^n$. A sequence  $(x_1,x_2,\dots,x_n)$ can be identified with the probability measure  $\mu=\sum_{i=1}^{n} \frac 1n \delta_{x_i}$. We denote the corresponding identification map by $\iota:X_d\to \Pp_{fin}(\Omega)$,
\beq\label{iota map}
\iota:(x_1,x_2,\dots,x_n)\to \sum_{i=1}^{n} \frac 1n \delta_{x_i}.
\eeq
Then a map $F:X_d\to X_{d'}$ that for any  $n$  maps a sequence $(x_1,x_2,\dots,x_n)$ of $d$-dimensional tokens to  a sequence
$(y_1,y_2,\dots,y_n)$ of $d'$-dimensional tokens so that the condition \eqref{preserves tokens} is satisfied,
defines a support-preserving map $f:\mathcal M^+_{fin}(\Omega) \to \mathcal M^+_{fin}(\RR^{d'})$ that is the zero-homogeneous extension of the map $f=\iota\circ F\circ \iota^{-1}:\Pp_{fin}(\Omega)\to \Pp_{fin}(\RR^{d'})$. 
%\label{subsec: in context maps}
This map satisfies $f(\sum_{i=1}^{n} \frac 1n \delta_{x_i})=
\sum_{i=1}^{n} \frac 1n \delta_{y_i}$.
%{\color{cyan} As the length of the sequence of tokens $n\to \infty$, the discrete measures converge to continuous and more general measures}
%and  thus by considering maps between measures one can understand the properties of the transformers for very
%long prompts.}
We consider the Wasserstein distance, which is a generalization of the permutation invariant distance of sequences of tokens. We recall that the 1-Wasserstein distance of the measures $\mu = \sum_{i=1}^{n} \frac 1n \delta_{x_i}$ and $\mu' = \sum_{i=1}^{n} \frac 1n \delta_{x'_i}$ is given by
\ba
W_1(\mu,\mu')=
\min_{\sigma} \frac 1n  \sum_{i=1}^{n} |x_i-x'_{\sigma(i)}|,
\ea
where the minimum is taken over the permutations, $\sigma:\{1,2,\dots,n\}\to \{1,2,\dots,n\}$. 
%{\color{cyan} As  $n\to \infty$, the discrete measures converge to continuous and more general measures.}
For background material on basic transformers, we refer the reader to Appendix \ref{sec:A.4}.
The convergence of the point measures toward continuous measures as $n \to \infty$,
is discussed in Appendix \ref{sec:A.3}.
% Theorem \ref{main thm} (see also the paragraph 2.2.1 on the informal definition of the regular part of the derivative) and Corollary \ref{cor:approximation-trans}.

%before considering the formal arguments below.}
% When the number, $n$, of tokens grows, the discrete measures $\sum_{i=1}^{n} \frac 1n \delta_{x_i}$ can converge to continuous measures (or mixtures of point measures, and measures that are singular with respect to the standard measure of $\RR^d$.) Thus to understand properties of transformers it is useful to consider mappings between ``general'' measures.
%}

\begin{lemma}\label{lem: general ai}
Let $f: \mathcal M^+(\Omega) \to \mathcal M^+(\RR^{d'})$  be a support-preserving map that is continuous in the 1-Wasserstein metric. Then, for any measure
of the form \eqref{finite sum 1} % \footnote{We can either write this formula again, or cite to earlier formula \eqref{finite sum}.} 
 %\beq\label{finite sum 1 again}
%\mu = \sum_{i=1}^{n} a_i \delta_{x_i} \in \mathcal M^+_{fin}(\Omega),
%\eeq
with $a_i>0$ we have that $f(\mu)$ is of the form \eqref{finite sum 2} and satisfies condition \eqref{preserves tokens}.
\end{lemma}

Lemma \ref{lem: general ai} can be proved by using sequences of points $x_i$ of which several are equal and simply approximating $a_i$ by rational numbers. The details of the proof  Lemma \ref{lem: general ai} are given in Appendix~\ref{app:lem: general ai}.

%Observe that when  $f: \mathcal M^+(\Omega) \to \mathcal M^+(\RR^{d'})$  is a support preserving map that 
%is continuous
% in the 1-Wasserstein metric, then the condition \eqref{finite sum 1} with arbitrary $a_i>0$
% \footnote{Condition $\sum_{i=1}^na_i=1$ is now removed.} implies 
% the equation \eqref{finite sum 2}. Indeed, to see this we can use a collection of points $x_i$ of which
% several are equal and approximate $a_i$ by rational numbers.}

\begin{definition}\label{def:in-context-maps}
We say that $f:\mathcal{M}^+(\Omega)\to \mathcal{M}^+(\R^{d'})$ is a support-preserving map given by an in-context map, $G :\ \mathcal{M}^+(\Omega) \times \Omega \to \R^{d'}$, if there exists such a map such that 
$$
f(\mu)=G(\mu)_\# \mu,
$$
where $G(\mu)$ is regarded as the map 
$
x \mapsto G(\mu)(x)=G(\mu,x)
$. We sometimes write $f=f_G$.
\end{definition}

Note that for a measure $\mu=\sum_{i=1}^{n} \frac 1n \delta_{x_i}$ 
it holds that $f_G(\mu)=
\sum_{i=1}^{n} \frac 1n \delta_{y_i}$, where $y_i=G(\mu,x_i).$
A particularly interesting example of such a map is $f_\Gamma:\mu\to \Gamma(\mu)_\# \mu
 $,
where the function $\Gamma$ is a multi-head self attention; see Section~\ref{ssub:4.1}.

In Corollary~\ref{cor:approximation-trans}, we revisit the connection between the maps in this definition and transformers. Our goal is to show that a support-preserving map $f$ under a ``certain'' smoothness condition can be written in the form, $f_G$, with an in-context map, $G$. In the following subsection, we specify this ``certain'' smoothness in detail.

\subsection{Regular part of the derivative}

%{\color{blue} Below, we assume that $\Omega\subset \RR^d$ is a compact set {\color{red} [Matti: to be checked where this is needed]}.}

\begin{definition} 
Let $ \eta,\rho>0$. We consider triplets $(\mu, x,\psi) \in \mathcal X$, with 
$$
  \mathcal X = \mathcal X_{\Omega, \rho, \eta}
  := \{ \mu \in \mathcal M^+(\Omega) : \mu(\Omega) \leq \rho \} \times \Omega \times 
  \{\psi\in  C^1_0(\mathbb{R}^{d'}):\ \mathrm{Lip}(\psi)\leq \eta\},
$$
which  is endowed with the distance function
\begin{equation}
\label{def X space}
D_{\mathcal X}\big((\mu_1, x_1,\psi_1),(\mu_2, x_2,\psi_2)\big) = W_1(\mu_1,\mu_2)+|x_1-x_2|+\|\psi_1-\psi_2\|_{L^\infty(\R^{d'})}.
\end{equation}
\end{definition}

We observe that $(\mathcal X,D_{\mathcal X})$ is not a complete metric space (i.e., the Cauchy sequences may not converge in ${\mathcal X}$), as the functions $\psi$ are assumed to be in the space $C^1_0(\RR^{d'})$, but we consider their convergence in $L^\infty(\RR^{d'})$. 

\begin{definition} \label{genearlized regular derivative}
Let $f:\mathcal{M}^+(\Omega)\to \mathcal{M}^+(\RR^{d'})$, and $(\mu, x, \psi) \in \mathcal X$.
We define the $L^\infty$-regular part of the Fr\'{e}chet derivative of $f$ at $(\mu,x,\psi)$ by the limit
\beq\label{A limit}
\overline {\mathcal D}_{f}(\mu,x,\psi):=
\lim_{k\to\infty }\lim_{\e \to +0 } \frac{\bra  \psi_k,  f(\mu_k + \epsilon \delta_x) - f(\mu_k) \ket}{\epsilon} 
\eeq
for all $\psi_k \in C^1_0(\RR^{d'})$
%{\color{blue} $\hbox{Lip}(\psi_k)\le \eta,$}
and $\mu_k\in \mathcal{M}^+(\Omega)$ such that 
\beq\label{locally constant}
\psi_k \hbox{ is constant in an open neighborhood of } \supp(f(\mu_k))
\eeq
and
$$
\lim_{n\to\infty}W_1(\mu_k,\mu)=0,\quad \lim_{n\to\infty} \|\psi_k-\psi\|_{L^\infty(\R^{d'})}=0.
$$
\end{definition}

We note that the existence of the limit $\overline {\mathcal D}_{f}(\mu,x,\psi)$ means that for all $\mu$ and $\psi$, the limits in \eqref{A limit} exist independently of the chosen sequences $\mu_k$ and $\psi_k$. 

\paragraph{2.2.1. Motivational observations.} Let $f_G$ be a support-preserving map given by in-context map $G:\mathcal{M}^+(\Omega)\times \Omega \to \R^{d'}$, where $(\mu,x) \mapsto G(\mu,x)$ is  continuous.
We observe that for $\mu \in \mathcal{M}^+(\Omega)$, $x \in \Omega$ and $\psi \in C^1_0(\RR^{d'})$,
\[
\frac{\bra  \psi,  f_G(\mu + \epsilon \delta_x) - f_G(\mu) \ket}{\epsilon} 
= \psi(G(\mu + \epsilon \delta_x, x)) 
+ \int \frac{\psi(G(\mu + \epsilon \delta_x, y)) - \psi(G(\mu, y)) }{\e} d\mu(y).
\]
Thus the limit as $\e \to +0$ can be written as a sum of two terms
$$
\lim_{\e \to +0} \frac{\bra  \psi,  f_G(\mu + \epsilon \delta_x) - f_G(\mu) \ket}{\epsilon} = D^{reg}_{f_G}(\mu,x,\psi)+D^{irreg}_{f_G}(\mu,x,\psi),
$$
where 
$$
D^{reg}_{f_G}(\mu,x,\psi):= \lim_{\e \to +0}
\psi(G(\mu + \epsilon \delta_x, x)) 
=
\psi(G(\mu,x))
$$
and (if the limit exists)
$$
D^{irreg}_{f_G}(\mu,x,\psi):=\lim_{\e \to +0} 
 \int \frac{\psi(G(\mu + \epsilon \delta_x, y)) - \psi(G(\mu, y)) }{\e} d\mu(y).
$$
We call $D^{reg}_{f_G}(\mu,x,\psi)$ the $L^\infty$-regular part of the Fr\'{e}chet derivative of $f_G$ and $D^{irreg}_{f_G}(\mu,x,\psi)$ the $L^\infty$-irregular part of the Fr\'{e}chet derivative. This terminology reflects the fact that $\psi \to D^{reg}_{f_G}(\mu,x,\psi)$ is continuous in the $L^\infty$-topology whereas $\psi \to D^{irreg}_{f_G}(\mu,x,\psi)$ is not. The lemma below states that $\overline {\mathcal D}_{f_G}(\mu,x,\psi)$ is an extension of the regular part of the derivative $D^{reg}_{f_G}(\mu,x,\psi)$ for functions $G$.

In what follows, we refer to the $L^\infty$-regular part of the Fr\'{e}chet derivative as the regular part of the derivative.  The following lemma is proved in  Appendix~\ref{app:lem: extension of derivative}.

\begin{lemma} \label{lem: extension of derivative}
Let $\Omega \subset \R^d$ be a compact set and let a support-preserving map be given by the in-context map $G:\mathcal{M}^+(\Omega) \times \Omega \to \R^{d'}$, where $(\mu,x) \mapsto G(\mu,x)$ is continuous. Then, for $(\mu, x, \psi) \in \mathcal X$, 
$$
\overline {\mathcal D}_{f_G}(\mu,x,\psi) = D^{reg}_{f_G}(\mu,x,\psi) = \psi (G(\mu,x))
$$
and the map $\mathcal X \ni (\mu, x, \psi) \mapsto \overline {\mathcal D}_{f_G}(\mu,x,\psi) \in \R$ is uniformly continuous with respect to the metric $D_{\mathcal X}$ defined in equation (\ref{def X space}).
\end{lemma}

As we see in Lemma~\ref{lem: extension of derivative}, for map $f_G$ defined with a uniformly continuous in-context function $G$, the regular part of derivative $\overline {\mathcal D}_{f_G}(\mu,x,\psi)$ coincides with the above defined object, $D^{reg}_{f_G}(\mu,x,\psi)$ on $\mathcal{X}$. So we consider $D^{reg}_{f_G}(\mu,x,\psi)$ as a new object that is different from the classical Fr\'{e}chet derivative, and show that the definition of $D^{reg}_{f_G}(\mu,x,\psi)$ can be extended as a generalized regular part of the derivative, $\overline{\mathcal D}_{f}(\mu,x,\psi)$, for a class of functions $f$, for which we do not assume that the classical Fr\'{e}chet derivative is well-defined. For further remarks on the regular part of derivative $\overline {\mathcal D}_{f}(\mu,x,\psi)$, see Appendix~\ref{app:remarks-reg}.

\section{Main result}

Our goal is to prove

\begin{theorem}\label{main thm}
Let $\Omega \subset \R^{d}$ be a compact set. Let $f :\ \mathcal{M}^+(\Omega) \to \mathcal{M}^+(\RR^{d'})$ be a continuous map in the 1-Wasserstein topology. Then, 
\begin{enumerate}
\item [(A1)] $f$ is a
% support preserving 
map given by some in-context map $G$ in the sense of Definition~\ref{def:in-context-maps}, i.e., $f = f_G$ with some function $G : \mathcal M^+(\Omega) \times \Omega \to \R^{d'}$; and 
\item [(A2)] the function $(\mu,x) \to G(\mu,x)$ is continuous, 
\end{enumerate}
if and only if
\begin{enumerate}
\item [(B1)] $f$ is a support-preserving map in the sense of Definition~\ref{def:support-preserving-maps}; and
\item [(B2)] the regular part of the derivative of $f$, $\overline {\mathcal D}_{f}(\mu,x,\psi)$, exists for all $ (\mu,x,\psi)\in \mathcal X$, and the map $\mathcal X \ni (\mu,x,\psi) \to \overline {\mathcal D}_{f}(\mu,x,\psi) \in \R$ is uniformly continuous with respect to the metric $D_{\mathcal X}$ given by Definition~\ref{def X space}.
\end{enumerate}

%\footnote{Now when we have removed the Lipschitz condition, we have to add  the proof on the equivalence of the continuity in A2 and uniform continuity in B2 (comment from Matti)}
Moreover, the map $(\mu,x) \to G(\mu,x)$ is Lipschitz if and only if the map $\mathcal X \ni (\mu,x,\psi) \to \overline {\mathcal D}_{f}(\mu,x,\psi) \in \R$ is a  Lipschitz map with respect to the metric $D_{\mathcal X}$.
\end{theorem}

Theorem~\ref{main thm} provides the characterization of support-preserving maps that can be represented by in-context maps through a push forward. Condition (B2) can be roughly described as the uniform continuity of a ``certain'' derivative of $f$, derived from Definition~\ref{genearlized regular derivative}. The continuity of $f$ is not sufficient for the theorem to hold as shown
in the following proposition, that is proved in Appendix~\ref{app:counterexample}.

\begin{proposition}\label{prop-non existance}
Let $d = 1$ and $\Omega = [-3,3] \subset \R$ and consider the set ${\mathcal P}(\Omega)$ endowed with the 1-Wasserstein topology. There exists a continuous, support-preserving map $f : {\mathcal P}(\Omega) \to {\mathcal P}(\Omega)$ such that there does not exist a continuous map $G:{\mathcal P}(\Omega)\times \Omega\to \Omega$ for which $f=f_G$. 
%Such a map, $f$, is given in formulas \eqref{Tb map}-\eqref{Tb map3} below. 
%This shows the importance of the assumptions on the derivative of the map $f$ in the main theorem.
\end{proposition}

\subsection{Sketch of the proof of Theorem \ref{main thm}: (A1)-(A2) imply (B1)-(B2)}
\label{sec:proof-1}

In this section,  we give a sketch of the main ideas of proof:
Assume that (A1) and (A2) hold true. Then, let $f_{G} :\ \mathcal{M}^+(\Omega)\to \mathcal{M}^+(\RR^{d'})$ be the map, $f_G(\mu)=G(\mu)_\# \mu$, with $G:\mathcal{M}^+(\Omega) \times \Omega \to \R^{d'}$. 
It is straightforward to prove (B1) and, hence, we will focus on proving that (B2) holds.
We assume that $\psi_k, \psi \in C^1_0(\R^{d'})$ and $\mu_k, \mu \in \mathcal M_+(\Omega)$, $k=1,2,\dots$ are sequences with 
$
\psi_k \hbox{ is constant in an open neighborhood of }\supp(f(\mu_k))
$
and
$
\lim_{k\to\infty}W_1(\mu_k,\mu)=0$, and 
$
\lim_{k\to\infty} \|\psi_k -\psi\|_{L^\infty(\R^{d'})}=0.
$ 
Let 
$$
\mu^\e_{k,x}:=\mu_k + \e \delta_x.
$$
Then, by a simple computation,
$$
\bra f_{G}(\mu^\e_{k,x}), \psi\ket
=
\int_{\R^{d}} \psi_k(G(\mu^\e_{k,x},y))d \mu_k(y) + \e \psi_k(G(\mu^\e_{k,x},x)).
$$
As the set $\Omega \subset \R^{d}$ is compact, 
$$
\mathcal M^+_\rho(\Omega) :=\{\mu \in \mathcal M^+(\Omega):\ \mu(\Omega)\le \rho\},
$$ is also compact by the Prokhorov’s theorem. Then the map $G:\mathcal M^+_\rho(\Omega)\times \Omega\to \RR^{d'}$
is uniformly continuous. As $G(\mu^\e_{k,x},\cdot)\to G(\mu_k,\cdot)$ uniformly in $\Omega \subset \RR^{d}$ as $\epsilon \to 0$,
we see that  
$$
\sup_{y \in \text{supp}(\mu_k)} |G(\mu^\e_{k,x},y)-G(\mu_{k},y)|\to 0\quad\hbox{as }\epsilon\to 0.
$$
Thus, we find that for sufficiently small $ \epsilon \in (0, 1)$
$$
\psi_k(G(\mu^\e_{k,x},y))=\psi_k(G(\mu_k,y))
$$
for all $y \in \supp(\mu_k)$, and
$$
\bra f_{G}(\mu^\e_{k,x}), \psi_k\ket
= \int_{\R^d} \psi_k(G(\mu_k,y))d \mu_k(y)+\e \psi_k(G(\mu_k,x)).
\nonumber
$$
This implies that  
$$
\overline {\mathcal D}_{f_G}(\mu_k,x,\psi_k) = \lim_{\e \to + 0} \frac{\bra f_{G}(\mu^\e_{k,x}), \psi_k\ket - \bra f_{G}(\mu_{k}), \psi_k\ket }{\epsilon}
= \psi_k(G(\mu_k,y)).
$$
Upon taking the limit $k \to \infty$, we obtain 
$$
\overline {\mathcal D}_{f_G}(\mu,x,\psi)
=\psi(G(\mu,y)).
$$
From the uniform (Lipschitz) continuity of $\psi$ and $G$, we can show that the regular part $\overline {\mathcal D}_{f_G}$ is uniformly (Lipschitz) continuous with respect to the metric $D_{\mathcal{X}}$. For the details of the proof, see Appendix~\ref{app:sec:proof-1}.

\subsection{Sketch of the proof of Theorem \ref{main thm}:  (B1)-(B2) imply (A1)-(A2)}
\label{sec:proof-2}

Again, here, we give a sketch of the main ideas of the proof. Assume that (B1) and (B2) hold true. Since $f$ is a support-preserving map, $f :\ \mathcal M^+(\Omega)\to \mathcal M^+(\RR^{d'})$, there are (possibly non-continuous) functions,
$$
    y_i:\Omega^n \times (0,\infty)^n \to \RR^{d'} ,\
    ({\boldsymbol{x}},{\boldsymbol a}) \to y_i({\boldsymbol x};{\boldsymbol a}),\quad
    i=1,2,\dots,n,
$$
where
$
\boldsymbol{x} = (x_1, \dots, x_n)
\quad \text{and} \quad
\boldsymbol{a} = (a_1, \dots, a_n),
$
such that the following holds: Let
$
    \mu = \sum_{i=1}^{n} a_i \delta_{x_i} \in \mathcal M_{fin}(\Omega),\quad a_i>0;
$
then the functions $y_i({\boldsymbol x};{\boldsymbol a})$ satisfy
\[
    f(\mu) = \sum_{i=1}^{n} a_i \delta_{y_i({\boldsymbol x};{\boldsymbol a})}.
\]
When $\mu \in \mathcal M^+_{fin,dif,(n)}(\Omega)$ (which is a refinement of the property that if $j\not =i$ then $a_j\not =a_i$),
% we see that by the definition of $\mathcal M^+_{fin,dif,(n)}(\Omega)$, that
the functions $({\boldsymbol x};{\boldsymbol a}) \to y_i({\boldsymbol x};{\boldsymbol a})$ must have the property that 
if $x_j=x_i$ then $y_j({\boldsymbol x};{\boldsymbol a}) = y_i({\boldsymbol x};{\boldsymbol a}).$

We have the following lemma, which is proved in Appendix~\ref{app-lemma yk}.  

\begin{lemma}\label{Lemma yk continuous}
Let $\mu_0=\sum_{i=1}^{n} a^0_i \delta_{x^0_i} \in \mathcal M^+_{fin,dif,(n)}(\Omega)$ and $\mu_p = \sum_{i=1}^{n} a^p_i \delta_{x^p_i} \in \mathcal M^+_{fin,(n)}(\Omega)$. Assume that for all $i=1,2,\dots, n$, it holds that $x^p_i \to x^0_i$ and  $a^p_i\to a^0_i$ as $p\to \infty$. Then it holds for all $j \in [n]$, that
$$
\lim_{p \to \infty} y_j({\boldsymbol x}^p;{\boldsymbol a}^p) = y_j({\boldsymbol x}^0;{\boldsymbol a}^0).
$$
\end{lemma}

We now return to the proof of Theorem \ref{main thm}. Let $\mu \in \mathcal M^+(\Omega)$ and $x \in \Omega$, and $\alpha \in C^\infty_0(\R^d)$ be a cutoff function
such that  $\alpha(x)=1$ for all  $x \in \Omega$ and $\hbox{Lip}(\alpha(x)\cdot x)\le \eta$. We define
$$
G(\mu,x):= 
\begin{pmatrix}
\overline {\mathcal D}_f(\mu, x, \alpha \pi_1)
\\
\vdots
\\
\overline {\mathcal D}_f(\mu, x, \alpha  \pi_{d'})
\end{pmatrix},
$$
where $\pi_\ell :\mathbb{R}^d \to \mathbb{R}$ is the projection $\pi_\ell(x)=x_\ell$ onto the $\ell$-th component. By (B2), the map $(\mu, x) \mapsto G(\mu,x)$ is continuous, which proves (A2). In what follows, we will prove (A1). 

When $\mu \in \mathcal M^+_{fin,dif,(n)}(\Omega)$, using Lemma~\ref{Lemma yk continuous}, we can prove that for each $j \in [n]$,
$$
G(\mu, x_j) = y_j({\boldsymbol x};{\boldsymbol a}),
$$
which is equivalent to
$$
f(\mu)=(G_\mu)_\# \mu\quad \text{ for } \mu \in \mathcal M^+_{fin,dif,(n)}(\Omega).
$$
For the case $\mu \in \mathcal M^+(\Omega)$, we choose the sequence $(\tilde \mu_k)_{m \in \N} \subset \mathcal M^+_{fin,dif}(\Omega)$ such that ${\tilde \mu_k} \to \mu$ as $k \to \infty$, where the limit is considered in the 1-Wasserstein topology (which is possible by Lemma~\ref{lem:dense}). We have already shown that for $\tilde \mu_k \in {\mathcal M}^+_{fin,dif}(\Omega)$,
$$
f({\tilde \mu_k})=(G_{{\tilde \mu_k}})_\#({\tilde \mu_k}).
$$
Hence, by the uniform continuity of $(\mu, x) \mapsto G(\mu, x)$, the limit $k \to \infty$ converges, 
$$
f(\mu)=\lim_{m\to\infty}(G_{{\tilde \mu_k}})_\#({\tilde \mu_k}) = (G_\mu)_\# \mu\quad \text{ for } \mu \in \mathcal M^+(\Omega).
$$
For the details of the proof, see Appendix~\ref{app:sec:proof-2}.

\section{Vlasov flows}

Here, we present the close connections between support-preserving maps satisfying (B1) and (B2) in Theorem~\ref{main thm}, Vlasov flows and measure-theoretic transformers.

\subsection{Infinitely deep measure-theoretic transformers: Universal approximation and the Vlasov equation}
\label{ssub:4.1}

An in-context map as it appears in a single-layer ``measure-theoretic'' transformer \cite{furuya2024transformers, castin2024smooth} based on multi-head self attention, is of the form,
$$
\Gamma : \Pp({\R^d}) \times {\R^d}
\to {\R^d},\,\,\,
\Gamma(\mu,x) := x + 
    \sum_{h=1}^H W^h
    \int_{\R^d} \frac{
        \exp\Big(
            \frac{1}{\sqrt{k}}
            \dotp{Q^h x}{K^h y}
        \Big)
    }{
        \int _{\R^d}
        \exp\Big(
            \frac{1}{\sqrt{k}}
            \dotp{Q^h x}{K^h z}
        \Big)
        d \mu(z)
    } V^h y\, d \mu(y),
% {\color{blue} = x+Att(\mu,x)},
$$
see  Appendix \ref{sec:A.4} for corresponding functions
operating to discrete measures and sequences of tokens.
Here, $K^h$ and $Q^h$ are the multi-head key and query matrices in $\mathbb{R}^{k \times d}$, $V^h$ are the multi-head value matrices in $\mathbb{R}^{d_{head} \times d}$, and $W^h$ are the multi-head weight matrices in $\mathbb{R}^{d \times d_{head}}$, respectively. By abuse of notation, $\Gamma(\mu)(x) = \Gamma(\mu,x)$ defines a map $\mathbb{R}^d \to \mathbb{R}^d$. For two in-context maps, $\Gamma_1$ and $\Gamma_2$, the composition $\Gamma_2 \diamond \Gamma_1$ is defined as 
\begin{align}\label{diamond composition}
    (\mu,x) \mapsto (\Gamma_2 \diamond \Gamma_1)(\mu,x) := \Gamma_2( \nu, \Gamma_1(\mu,x)),\quad \nu := \Gamma_1(\mu)_\sharp \mu.
\end{align}
With this composition, the in-context map, $G_{\rm tran}$ say, for a multi-layer measure-theoretic transformer is obtained. To be precise, the composition should alternate between in-context maps and context-free MLPs, $F(\mu,x) =
 F(x)$ say. When restricted to finite discrete empirical measures of the form $\tfrac{1}{n} \sum_{i=1}^n \delta_{x_i}$, $f_{\rm tran} := f_{G_{\rm tran}}$ (cf.~Definition~\ref{def:in-context-maps}) reduces to a classical transformer acting on a sequence of tokens, $(x_1,\dots,x_n)$ rather than on a measure $\mu$. For more details, see \cite[Section 2]{furuya2024transformers}. Being based on multi-head self attention, $G_{\rm tran}$ is (locally) Lipschitz \cite{castin2024smooth}, and, hence, satisfies (A2) in Theorem~\ref{main thm}.

% \begin{lemma} \label{lem:self-B-hold}
% The measure-theoretic self-attention $f_{\mathrm{self}}$ defined by \eqref{eq:self} satisfies (B1) and (B2). 
% \end{lemma}

% The measure-theoretic transformer with (multi-head) self-attention, $f_{self}$ say, is given by the push forward, 
% \begin{equation} \label{eq:self}
% f_{\mathrm{self}}(\mu) = \Gamma(\mu)_{\sharp}(\mu), \quad \mu \in \Pp({\R^d}).
% \end{equation}

% Using this notation, we define the measure-theoretic transformer $f_{\mathrm{tran}}$ by
% \begin{equation} \label{eq:measure-theoretic-transformer}
% f_{\mathrm{tran}}(\mu):=
%     (F_{L} \diamond \Gamma_{L} \diamond \ldots \diamond
%     F_{1} \diamond \Gamma_{1}(\mu, \cdot))_{\sharp}\mu.
% \end{equation}
% where $\Gamma_\ell$ are in-context self-attentions, and $F_{\ell}:{\R^d}\to {\R^d}$ are "context-free" MLPs, that the output is independent of context $\mu$, i.e.,
% $$
% F_{\xi}(\mu, x) = F_{\xi}(x).
% $$

As a consequence of Theorem~\ref{main thm}, $f_{\rm tran}$ satisfies (B1) and (B2), while using \cite[Theorem 1]{furuya2024transformers}, we obtain the following universal approximation result
that is prove in  Appendix~\ref{app:approximation-trans}.

\begin{corollary}
\label{cor:approximation-trans}
Let $f : \mathcal M^+(\Omega) \to \mathcal M^+(\RR^d)$ satisfy (B1) and (B2) in Theorem~\ref{main thm}. Then, for any $\epsilon \in (0,1)$, there exists a sufficiently deep measure-theoretic transformer, $f_{\rm tran}$, (that is, a deep composition of multi-head self attention maps and MLPs), such that 
\[
\sup_{\mu \in \Pp(\Omega)} W_1(f_{\rm tran}(\mu), f(\mu)) \leq \epsilon.
\]
\end{corollary}

%\medskip

Next, we consider a MLP $F_\eta:\R^d\to \R^d$,
 see \eqref{recall FLP} in Appendix
 \ref{sec:A.4}, and 
the attention function
%\label{composition and map mathcal V}
$$
%\mathcal{V} 
\hbox{Att}_\xi
: \Pp({\R^d}) \times {\R^d}
\to {\R^d},\quad 
%\mathcal{V}(\mu,x) 
 \hbox{Att}_\xi(\mu,x):= 
    \sum_{h=1}^H W^h
    \int_{\R^d} \frac{
        \exp\Big(
            \frac{1}{\sqrt{k}}
            \dotp{Q^h x}{K^h y}
        \Big)
    }{
        \int _{\R^d}
        \exp\Big(
            \frac{1}{\sqrt{k}}
            \dotp{Q^h x}{K^h z}
        \Big)
        d \mu(z)
    } V^h y d \mu(y).
% {\color{blue} = x+Att(\mu,x)}.
$$
where $\eta$ and $\xi$ are sets of parameter matrices for MLPs and the attention, respectively. 
Let us write the MLP $F_\eta$ as
$F_\eta=Id_x+H_\eta$, and define
$\mathcal V=\hbox{Att}_\xi+H_\eta\circ (Id_x+\hbox{Att}_\xi)$,
so that $F_\eta(\Gamma_\xi(\mu,x))=x+\mathcal V(\mu,x)$,
see formulas 
\eqref{composition and map mathcal V} and \eqref{map V} in Appendix
 \ref{app:transformers for discrete measures} for detailed formulas.
Again, by abuse of notation,
$\mathcal{V}(\mu)(x) = \mathcal{V}(\mu,x)$ defines a map or vector field, $\mathbb{R}^d \to \mathbb{R}^d$. 
We consider layers, $x_i(\tau+1)=F_{\eta_\tau}(\Gamma_{\xi_\tau}(\mu_i(\tau),x_i(\tau)))$, 
where the sets $\eta_\tau$ and $\xi_\tau$ of parameter matrices depend on $\tau$. Then, we find that 
\begin{equation*}
x_i(\tau+1)-x_i(\tau)=
\mathcal{V}_{\tau}(\mu_{\tau})(x_i(\tau)),\quad\hbox{where $\mu_\tau(.)=\tfrac{1}{n} \sum_{i=1}^n \delta_{x_i(\tau)}(.)$ and $\tau=0,1,2,\dots,T$.}
\end{equation*} 
Taking the continuum limit, scaling $\mathcal{V}_\tau$ with $1/T$, where $T$ signifies the number of layers, and identifying the layer index, $\tau$, with $t \in [0,1]$ that corresponds to the limit of values $\tau/T$ as $T \to \infty$, the tokens that evolve according to an infinitely deep transformer
%{\color{cyan} [FOLLOWING TEXT IS TO BE REMOVED: { \textit{but omitting the above mentioned MLPs}, }} 
satisfy
\begin{equation} \label{eq:discr-mf-trans}
\dot{x}_i(t) = \mathcal{V}_t(\mu_t)(x_i(t))
\end{equation}
for all $i \in [n]$, where $\mu_t(.) = \tfrac{1}{n} \sum_{i=1}^n \delta_{x_i(t)}(.)$,  and  $t \in [0,1]$; see also \cite{Zhong2022}. This is extended to positive measures by the partial differential, nonlocal transport equation,
\begin{eqnarray}
\partial_t \mu_t + \operatorname{div}(\mathcal{V}_t(\mu_t) \mu_t) &=& 0\quad
\text{on } [0,1] \times \mathbb{R}^d,
\\
\mu_t|_{t=0} &=& \mu_0\quad\text{ on } \mathbb{R}^d
\end{eqnarray}
in the sense of distributions, replacing the (neural) ODE in (\ref{eq:discr-mf-trans}); see \cite{RenardyRogers2006}. It basically follows from
\begin{eqnarray}
\frac{\operatorname{d}}{\operatorname{d}\!t} \,
\int_{\mathbb{R}^d} \varphi(t,x) d\mu_t(x)
&=& \frac{\operatorname{d}}{\operatorname{d}\!t} \,
\frac{1}{n} \sum_{i=1}^n \varphi(t,x_i(t))
% \nonumber\\
% &=& \frac{1}{n} \sum_{i=1}^n \partial_t \varphi(t,x_i(t))
% + \langle \nabla_x \varphi(t,x_i(t)),\mathcal{V}_t(\mu_t)(x_i(t)) \rangle
\nonumber\\
&=& \int_{\mathbb{R}^d} (\partial_t \varphi(t,x)
+ \langle \nabla_x \varphi(t,x),\mathcal{V}_t(\mu_t)(x) \rangle) \,
d\mu_t(x)
\end{eqnarray}
for all $\varphi \in C^{\infty}_c([0,1] \times \mathbb{R}^d)$, and integrating by parts. Thus, an infinitely deep measure-theoretic transformer without MLPs, with $\mu_t := f^{\infty}_{{\rm tran};t}(\mu_0)$, $t \in (0,1]$, is argued to satisfy the Vlasov equation; see \cite{piccoli2015control} and \cite{paul2025}. Some prior work \cite{sander2022sinkformers, geshkovski2025mathematical, castin2025unified} already discussed that the mean-field (with respect to tokens) and deep transformers are associated with nonlocal transport PDEs. Moreover, the infinitely deep in-context map, $G^{\infty}_{{\rm tran};t}$, satisfies an evolution equation in spacetime that generalizes the equation \eqref{eq:discr-mf-trans} for the point measures,
\[
\partial_t G^{\infty}_{{\rm tran};t}(\mu_0,x)
= \mathcal{V}_t(\mu_t)(G^{\infty}_{{\rm tran};t}(\mu_0,x)),\quad
G^{\infty}_{{\rm tran};0}(\mu_0,x) = x.
\]

\subsection{The solution map of the Vlasov equation satisfies (B1) and (B2) of Theorem~\ref{main thm}}
\label{sec:sol-nonlocal}

\cite{piccoli2015control} studied the well-posedness of nonlocal transport PDEs having the form,
\begin{equation} \label{eq:nonlocal-PDE}
\partial_t \mu_t + \mathrm{div}(V(t,\mu_t) \mu_t) = 0, \quad \mu_t|_{t=0} = \mu_0,
\end{equation}
where $\mu = \mu_t = \mu(t)$ is a time-depending probability measure on $\R^d$ and 
$V(.,\mu) : \R \times \R^d \to \R^d$ is a $C^1-$smooth vector field that depends on $(t,x) \in \R \times \R^d\to \R^d$ and the measure $\mu$. The vector field $V(\mu)$ is called the velocity field. 
% This includes the kinetic Cucker-Smale equation \cite{wang2023pattern}. 

Under the assumptions on $V$ required by \cite[Theorem~2.3]{piccoli2015control}, there exists a unique solution of \eqref{eq:nonlocal-PDE}.  Moreover, the solution at time $t$, $\mu_t$ can be written as 
\[
\mu_t = G_t(\mu_0)_\sharp \mu_0, 
\]
where $G_t$ is defined as the unique solution of the following Cauchy problem,
\begin{equation*}
\partial_t G_t(\mu_0, x) = V(t,\mu_t)(G_t(\mu_0, x)),\quad G_0(\mu_0, x) = x.
\end{equation*}
Thus, we can define the solution map $f_T : \Pp(\Omega) \to \Pp(\R^{d})$ (the solution at time $t=T$) by
\begin{equation} \label{eq:sol-map-nonlocal}
f_T(\mu_0) := \mu_T. 
\end{equation}

%\medskip

The map, $f_T$, is a support-preserving map given by the in-context map, $G_T$. 
The following proposition is proved Appendix~\ref{app-prop-f-T-hold-B}.

\begin{proposition} \label{prop-f-T-hold-B}
Under the assumptions for $V$ required by \cite[Theorem 2.3]{piccoli2015control}, the solution map, $f_T$, defined by \eqref{eq:sol-map-nonlocal} satisfies (B1) and (B2). That is,
the solution map $f_T$ of the Vlasov flow can be represented as a map $ f_{G_T}:\mu\to G_T(\mu)_\#\mu$ with a continuous in-context map $G_T$.
\end{proposition}

% \begin{remark}
% Note that if we add some diffusion term in \eqref{eq:nonlocal-PDE}, its solution map cannot be written as the pushforward. This observation tell us that transformer is contrastive to the diffusion models \cite{yang2023diffusion}. 
% \end{remark}

% %%
% %%
% %%
% %%
% %%

% \paragraph{Map from prior to posterior (Negative examples)}

% Let $\mu \in \Pp(\Omega)$ be prior, and let $y$ be observation, and let $\mu^y \in \Pp(\Omega)$ be posterior. Assume that the posterior density is given by 
% \[
% \frac{d \mu^y}{d\mu}(x) = \frac{1}{Z_{\mu}} \exp{(-\Phi(x;y))}, \quad Z_\mu = \int  \exp{(-\Phi(x;y))} d\mu(x).
% \]
% The typical application to estimate the posterior is the Bayesian inverse problem \cite{stuart2010inverse, dashti2015bayesian}, where potential function $\Phi$ depends on the forward map and observation of the inverse problem. 
% Recently, learning the map from prior to posterior might be concerned in the field of Bayesian inverse problems, arising from the question if there is robustness of the posterior with respect to the perturbations in the prior. \cite{sprungk2020local, cvetkovic2025upper, garbuno2023bayesian} provided the several type of stability estimates for prior-to-posterior maps. 
% However, in general, the prior-to-posterior map $f_y : \mu \mapsto \mu^y$ (fix observation $y$) is not a support preserving map since $f$ does not always map the empirical uniform measure to the uniform measure. 

\section{Conclusion and Discussion}

In this work, we fully characterize mappings between measures that can be universally approximated by measure-theoretic transformers. To this end, we introduce a ``certain'' smoothness condition, which roughly entails that a ``certain'' derivative of the mapping is uniformly continuous. A limitation of our method is that it is not quantitative. We make rigorous a connection between particle systems and mappings between measures through measure-theoretic transformers in the mean-field regime, which connection has been discussed in various works before. This has implications in the framing of LLMs. Beyond the Vlasov equation, it will be interesting to study the BBGKY hierarchy describing the dynamics of a system of a large number of interacting particles (see, for example, \cite{golse2016}) with measure-theoretic transformers.

\section*{Acknowledgements}

T.~F. was supported by JSPS KAKENHI Grant Number JP24K16949, 25H01453, JST CREST JPMJCR24Q5, JST ASPIRE JPMJAP2329. M.~L. was partially supported by the Advanced Grant project 101097198 of the European Research Council, Centre of Excellence of Research Council of Finland (grant 336786) and the FAME flagship of the Research Council of Finland (grant 359186). Views and opinions expressed are those of the authors only and do not necessarily reflect those of the funding agencies or the EU. M.V. de H. initiated the work presented here while he was an invited professor at the Centre Sciences des Données - de l'Ecole Normale Supérieure, Paris. He acknowledges the support of the Department of Energy under grant DE-SC0020345, Oxy, and the corporate members of the Geo-Mathematical Imaging Group at Rice University. The authors would like to express their gratitude to G. Peyré and E. Trélat for their insightful discussions and valuable perspectives.

\bibliographystyle{plain}
\bibliography{iclr2025_conference,ref}

\newpage

\appendix

\section{Notation and a summary of results of measure theory}
\label{app:notations}

\subsection{Notations} \label{sec:A.1}

Let $\Omega \subset \RR^d$ be a compact set. We denote by $\mathcal{P}(\Omega)$ the space of probability measures on $\Omega$.
Below, all measures $\mu$ on subset $\Omega$ of $\RR^d$ are defined 
on the $\sigma$-algebra of the Borel sets of $\Omega$. We denote by $C(\Omega)$ the space of continuous functions from $\Omega$ to $\R$,
and the dual coupling between $\varphi \in C(\Omega)$ and $\mu \in \Pp(\Omega)$ by 
$$
\dotp{\varphi}{\mu}
:= \int_\Omega \varphi(x) \mathrm{d} \mu(x).
$$
With the weak$^*$ topology on $\Pp(\Omega)$, we have the convergence of sequences of measures,
$$
\mu_k \rightharpoonup^* \mu
\quad
\Leftrightarrow
\quad 
\Big( 
\forall \varphi \in C_0(\Omega), \: \dotp{\varphi}{\mu_k}\, \to \dotp{\varphi}{\mu}
	\Big).
$$
In the case when $\Omega$ is compact, the weak $^*$ topology is equivalent to the topology of the Wasserstein distance $W_p$ ($1 \leq p < \infty$), meaning that 
$$
	\mu_k \rightharpoonup^* \mu
	\quad
	\Leftrightarrow
	\quad 
	W_p(\mu_k,\mu) \to 0, 
$$
see e.g., \cite[Theorem 5.10]{santambrogio2015optimal}. By the duality theorem of Kantorovich and Rubinstein,  when $\mu, \nu \in \Pp(\Omega)$, where $\Omega$ is compact, we have that 
$$
W_{1}(\mu ,\nu )=\sup \left\{\left. 
\int_{\Omega}\varphi(x)\,\mathrm {d} (\mu -\nu )(x)\,\right| \varphi: \Omega \to \mathbb {R} {\text{ continuous}},\ \mathrm{Lip}(\varphi) \leq 1 \right\},
$$
where 
$$
\mathrm{Lip}(\varphi)
:= 
\sup_{x \neq y}
\frac{|\varphi(x)-\varphi(y)|}{|x-y|}
$$
denotes the Lipschitz constant for $\varphi : \Omega \to \RR$. 

We extend $\mathcal{P}(\Omega)$, that is, the set of all probability measures to the set of all strictly positive, finite measures 
$$ 
\mathcal M^+(\Omega)=\{s \mu :\ \mu \in \Pp(\Omega), \ s>0  \}.
$$
We also extend the $W_1$ distance to $\mathcal M^+(\Omega)$  by defining for  $\mu_1,\mu_2 \in \Pp(\Omega) $ and $s_1,s_2>0$ 
$$
 W_1(s_1\mu_1,s_2\mu_2)= W_1(\mu_1,\mu_2)+|s_1-s_2|,
$$
see \cite{extensionWasserstein}.
Using this extension, we can extend the map $f : \Pp(\Omega) \to \Pp(\RR^{d'})$ to a map between positive measures invoking $m$-homogeneity ($m \in \mathbb{N}_0$) according to
$$ 
f(s\mu)=s^m f(\mu)\quad \text{ for all }s\in \RR_+.
$$
We write
$$
\mathcal M^+_{fin}(\Omega) := \left\{\sum_{i=1}^n a_i\delta_{x_i} \in \mathcal M^+(\Omega):\ x_i \in \Omega,\ a_i > 0,\ n\in \NN \right\},
$$
and
$$
\mathcal M^+_{fin,(n)}(\Omega) := \left\{\sum_{i=1}^n a_i\delta_{x_i} \in \mathcal M^+(\Omega):\ x_i \in X,\ a_i > 0
 \right\}.
$$
Finally, we denote by $\mathcal M^+_{fin,dif,(n)}(\Omega)$ the measures of the form 
$$
    \mu = \sum_{i=1}^{n} a_i \delta_{x_i} \in \mathcal M_{fin,(n)}(\Omega),
$$   
where $a_j> 0$ and for all non-empty subsets $J,K \subset \{1,2,\dots,n\}$ satisfying $J\cap K=\emptyset$ it holds that
$$
\sum_{j \in J} a_j \not= \sum_{k \in K} a_k.
$$
We set $\mathcal M^+_{fin,dif}(\Omega)=\bigcup_{n=1}^\infty \mathcal M^+_{fin,dif,(n)}(\Omega).$
For $\mu\in \mathcal M^+_{fin,dif,(n)}(\Omega)$ we define the minimal gap
\beq\label{gap finite with difference}
\operatorname{gap}(\mu) = \min_{J,K\subset\{1,2,\dots,n\},J\cap K=\emptyset,\ J\not =\emptyset}
\left|\sum_{j \in J} a_j -\sum_{k \in K} a_k\right|.
\eeq

\subsection{Push forwards of measures}
\label{sec:A.2}

We will consider push forwards of measures in various maps. When $\nu$ is a general Borel measure 
on set $\Omega \subset \R^d$ and $F: \Omega \to \R^{d'}$ is a continuous map, the push-forward measure of $\nu$ in the map $F$, denoted by $F_{\#}\nu$, is the measure that for an open (or Borel measurable) set $A$ is defined to be
$$
F_{\#}\nu(A)=\nu(F^{-1}(A)).
$$
When $\mu=\sum_{j=1}^n a_j\delta_{x_j}$ is a discrete measure supported at points $x_1,\dots,x_n$,
we have
\ba
F_{\#}\mu= \sum_{j=1}^n a_j\delta_{y_j},\quad y_j=F(x_j).
\ea
When $\nu =\rho(x) dx$ is a continuous measure where $\rho:\R^d\to [0,\infty)$ is a continuous function and $dx$ is the Lebesgue measure on $\R^d$ and $F:\R^d\to \R^d$ is a differentiable map which inverse function $F^{-1}:\R^d\to \R^d$ is differentiable, then
\ba
F_{\#}(\rho(x)dx)=\tau(x)dx,\quad \hbox{where } \tau(x)=\rho(F^{-1}(x))\cdot \bigg|\det\bigg(\frac {\p F}{\p {x}}(F^{-1}(x))\bigg)\bigg|, 
\ea
where $\det(\frac {\p F}{\p {x}}(F^{-1}(x)))$ is the determinant of the Jacobian matrix of the function $F$ evaluated at the point $F^{-1}(x)$.

When $F:\R^d\to \R^{d'}$ is a smooth injective map
and $d'>d$, the push forward of the measures $\mu$ on
$\R^d$
to the $d$-dimensional image manifold $M=F(\R^d)$ of $F$
are discussed e.g. in \cite{Kothari2021TrumpetsIF}.

\subsection{Convergence of point measures to a general measure}
\label{sec:A.3}

Let us consider the convergence of discrete measures $\mu_n=\sum_{j=1}^n a_{n,j} \delta_{x_{n,j}}$ to continuous measures.
Let $\Omega\subset \R^d$ be a compact set, $x_{n,j}\in \Omega,$ and $a_{n,j}>0$ are such that $\sum_{j=1}^{n} a_{n,j}=1$. If for all relatively open subsets $U\subset \Omega$  there exists limits 
\beq\label{limits}
m(U)=\lim_{n\to \infty}  \mu_{n}(U),\quad \hbox{where } \mu_{n}(U)=\sum_{x_{n,j}\in U} a_{n,j},
\eeq
then  the limits $m(U)$ define a (Borel)  probability measure in  $m\in \mathcal P(\Omega)$ and
the measures $\mu_n$ converge in the 1-Wasserstein topology to the measure $m$.

By the Portmanteau theorem, see \cite{Klenke}, Theorem 13.16 (see also Remark 13.14),
the existence of limits \eqref{limits} is equivalent to following conditions:
\begin{enumerate}
    \item [(C1)] There is a probability measure $m\in \mathcal P(\Omega)$
such that $m(U)\ge \liminf_{n\to \infty}  \mu_{n}(U)$ for all relatively open sets $U\subset \Omega$

\item [(C2)] There is a probability measure $m\in \mathcal P(\Omega)$
such that for all  Lipschitz functions $\phi:\Omega\to \R$ 
\beq
\int_\Omega \phi d\mu_n=\sum_{j=1}^n  a_{n,j}\phi(x_{n,j})\to \int_\Omega \phi dm,\quad\hbox{as }n\to \infty,
\eeq
\end{enumerate}

that is, the existence of limits $m(U)$ in \eqref{limits} and  the conditions (C1) and (C2) are all equivalent to that 
$\mu_n$ converge weakly to $m$ that is further equivalent to that $\mu_n$ converge to $m$ in the 1-Wasserstein topology.

In particular, consider the case when $x_{n,j}=x_j$ are independent of $n$ and $a_{n,j}=1/n$.
Also, let us consider  the Lipschitz functions $\phi:\Omega\to \R$ as feature functions.
That is the measures, $\mu_n$ are
the point measures 
$$
\mu_n=\sum_{j=1}^n \frac 1n \delta_{x_{j}}
$$
that correspond to  prompts  $X_n=(x_1,x_2,\dots,x_n)$, that is,  sequences  of $n$ tokens.
Then, if the the prompt length $n$ goes to infinity, if follows from Prokhorov's theorem 
\cite[Theorem 13.29 and Corollary 13.30]{Klenke}, that there is at least one  sub-sequence  $X_{n_k}$ of prompts, where $n_k\to \infty$ as $k\to \infty$
such that for all feature functions $\phi\in C^{0,1}(\Omega)$ the averages of the features 
$$
\int_\Omega \phi d\mu_{n_k}=\sum_{j=1}^{n_k} \frac 1n \phi(x_{j})
$$
converge to some limit 
$$
\lim_{k\to \infty}\int_\Omega \phi d\mu_{n_k}=\int_\Omega \phi d\mu,
$$
These limits define a probability measure $\mu\in \mathcal P(\Omega)$ such that
\beq\label{subseq of measures}
\lim_{k\to \infty} W_1(\mu_{n_k},\mu)=0.
\eeq
Moreover, by \cite[Theorems I.13 and I.14]{ReedSimon1}, the measure $\mu$ can be written
as a sum of three measures,
\beq\label{3 measures}
\mu=\nu_1+\nu_2+\nu_3,\quad \nu_1=\sum_{i=1}^N a_j\delta_{y_j},\quad \nu_2=\rho(x)d{\bf x},\quad \nu_3\perp d{\bf x}
\eeq
where $\nu_1$ is a pure point measure supported at the points $y_j\in \Omega$ with $N\in \mathbb N\cup \{\infty\}$ and $a_j>0$,
$\nu_2$ is an absolutely continuous measure having the density $\rho(x)$ with respect to the Lebesgue measure $d{\bf x}$ of $\R^d$,
and $\nu_3$ is a singular continuous measure, that is, there is a set $S\subset \Omega$ which
the Lebesgue measure is zero such that $\nu_3(\Omega\setminus S)=0$ and
$\nu_3(\{p\})=0$ for all singleton sets with $p\in \Omega.$

\subsection{Attention and transformers}
\label{sec:A.4}
Finally we recall notations related to attention functions.
The  multi-head self attention is the function
\beq\nonumber
& &\Gamma : \Pp({\R^d}) \times {\R^d}
\to {\R^d},\\  \label{general transfomer}
& &\Gamma(\mu,x) = x + 
    \sum_{h=1}^H W^h
    \int_{\R^d} \frac{
        \exp\Big(
            \frac{1}{\sqrt{k}}
            \dotp{Q^h x}{K^h y}
        \Big)
    }{
        \int _{\R^d}
        \exp\Big(
            \frac{1}{\sqrt{k}}
            \dotp{Q^h x}{K^h z}
        \Big)
        d \mu(z) 
    } V^h y\, d \mu(y)\\ \nonumber
 &&\hspace{11mm}= x+Att(\mu,x).
\eeq
We recall that here  $K^h$ and $Q^h$ are the multi-head key and query matrices in $\mathbb{R}^{k \times d}$, $V^h$ are the multi-head value matrices in $\mathbb{R}^{d_{head} \times d}$, and $W^h$ are the multi-head weight matrices in $\mathbb{R}^{d \times d_{head}}$, respectively. 
When 
\begin{align}\label{mu measure1}
\mu=\sum_{i=1}^n\frac 1n \delta_{x_i}
\end{align} 
is a discrete measure
corresponding to a sequence $x_1,x_2,\dots,x_n$ of points in $\Omega\subset \R^d$,
it holds that 
\begin{align}\nonumber
    \Gamma(\mu,x) &=
    %&= 
    x + 
    \sum_{h=1}^H W^h
    \int_{\R^d} \frac{
        \exp\Big(
            \frac{1}{\sqrt{k}}
            \dotp{Q^h x}{K^h y}
        \Big)
    }{
        \int _{\R^d}
        \exp\Big(
            \frac{1}{\sqrt{k}}
            \dotp{Q^h x}{K^h z}
        \Big)
        d \mu(z)
    } V^h y\, d \mu(y)
     \\ \label{discrete transfomer}
    &= x + 
    \sum_{h=1}^H W^h
    \sum_{\ell=1}^n \frac{
        \exp\Big(
            \frac{1}{\sqrt{k}}
            (Q^h x)^\top (K^h x_\ell)
        \Big)
    }{
        \sum_{j=1}^n 
        \exp\Big(
            \frac{1}{\sqrt{k}}
            (Q^h x)^\top (K^h x_j)
        \Big)
    } V^h x_\ell\, ,
  %  \\ \nonumber
  %  &= x + \hbox{Att}(\mu,x),
\end{align}   
where $v^\top$ denotes the transpose of a column vector $v\in \R^k$.

For the measure $\mu$ given in \eqref{mu measure1} it holds that
\begin{align}\label{map F Gamma}
    f_\Gamma \bigg(\sum_{i=1}^n\frac 1n \delta_{x_i}\bigg) &=\Gamma(\mu,\cdot)_\# \mu= \sum_{i=1}^n \frac 1n \delta_{y_i},
   \end{align}
where $y_i=\Gamma(\mu,x_i)$. 

In the case that the measures $\mu_{n_k}=\sum_{i=1}^{n_k}\frac 1n \delta_{x_i}$ converge in 1-Wasserstein topology to a measure
$\mu$ as $k\to \infty,$ we have pointwise limits
\beq 
\lim_{k\to \infty}\Gamma(\mu_{n_k},x)=\Gamma(\mu,x),
\eeq
where $\Gamma(\mu_{n_k},x)$ and $\Gamma(\mu,x)$ are given in formulas \eqref{discrete transfomer} and \eqref{general transfomer},
respectively.
Moreover, it holds that  the push forwards of the measures satisfy the limit
\beq 
\lim_{k\to \infty}\Gamma(\mu_{n_k},\cdotp)_\#\mu_{n_k}=\Gamma(\mu,\cdotp)_\#\mu
\eeq
in the 1-Wasserstein topology.

Let us next consider the  prompts $(x_1,x_2,\dots,x_n)$ and 
the corresponding discrete measures $\mu_n=\sum_{i=1}^{n} \frac 1n \delta_{x_i}$.  
As seen above, then there exists at least one  sub-sequence $\mu_{n_k}$  
that
converge to a general probability measure $\mu\in {\mathcal P}(\Omega)$,
that is a sum of a point measure,  a  continuous measure, and a measure that are singular with respect to the standard measure of $\RR^d$, see formula \eqref{3 measures}. 
Thus, to understand properties of transformers  it is useful to consider mappings between general probability measures 
that have the same properties of the  transformers.

\section{Proofs for technical parts}

\subsection{Proof of Lemma~\ref{lem:dense}}
\label{app:dense}
\begin{proof}
Let $\mu \in \mathcal M^+(\Omega)$ and let $\epsilon \in (0,1)$. Since $\mathcal M^+_{fin}(\Omega)$ is dense in $\mathcal M^+(\Omega)$ in $1$-Wasserstein topology, there is $\mu_k \in \mathcal M^+_{fin,dif}(\Omega)$ with 
$
\mu_k = \sum_{i=1}^{n} a_i \delta_{x_i}, a_i>0
$ 
such that 
$$
W_1(\mu_k, \mu) \leq \epsilon.
$$
We can choose 
$\tilde a_1,\dots,\tilde a_n >0$ such that $|\tilde a_j-a_j|<\epsilon/n$ and, for any 
non-empty disjoint subsets $J,K\subset\{1,\dots,n\}$, it holds that 
\[
\sum_{i\in J}\tilde a_i \neq \sum_{i\in K}\tilde a_i.
\]
Indeed, setting $\tilde a_i = a_i+\eta_i$, the equality
\[
\sum_{i\in J}\tilde a_i = \sum_{i\in K}\tilde a_i,
\]
is equivalent to
\[
\sum_{i\in J}\eta_i - \sum_{i\in K}\eta_i 
= \underbrace{\sum_{i\in K}a_i - \sum_{i\in J}a_i}_{=: \Delta_{J,K}}.
\]
Since the set $\cup_{J,K}\{ \eta \in \RR^n : \sum_{i\in J}\eta_i - \sum_{i\in K}\eta_i =  \Delta_{J,K} \}$ of affine hyperplanes are measure-zero set, we can choose small $|\eta_i| < \epsilon/n$ so that 
\[
\eta \notin \bigcup_{J,K}\left\{ \eta \in \RR^n : \sum_{i\in J}\eta_i - \sum_{i\in K}\eta_i =  \Delta_{J,K} \right\}.
\]
Thus, defining by $\tilde{\mu}_n = \sum_{i=1}^{n} \tilde{a}_i \delta_{x_i} \in \mathcal M^+_{fin,dif}(\Omega)$, we see that 
\[
W_1(\mu_k, \tilde{\mu}_n) < \epsilon.
\]
We have proved Lemma~\ref{lem:dense}.
\end{proof}

\subsection{Proof of Lemma~\ref{lem: general ai}}
\label{app:lem: general ai}

\begin{proof}
When  $\tilde x_j\in \Omega,$ $\tilde a_j>0$, $j=1,2,\dots,\tilde n$ are of the form
$\tilde a_j=cm_j$ where $c>0$ and $m_j\in \mathbb Z_+$, we can can write the measure
\beq
\mu = \sum_{j=1}^{\tilde n} \tilde  a_j \delta_{\tilde x_j}
\eeq
in  the form
\beq
\mu = \sum_{i=1}^{n} \frac {\mu(\Omega)}{n}\delta_{x_i},
\eeq
where $n=\sum_{j=1}^{\tilde n}m_k$ and $x_1,x_2\dots,x_n$ is a sequence
where each point $\tilde x_j$ appears $m_j$ times. As $f$ is a support preserving map, there are $y_i\in\R^{d'}$,
$i=1,2,\dots,n$, such that
\beq
f(\mu) &=&
\sum_{i=1}^{n} \frac {\mu(\Omega)}{n}\delta_{y_i}.
\eeq
Moreover, $y_{i_1}=y_{i_2}$ if $x_{i_1}=x_{i_2}$. Hence, we can write $f(\mu)$ in the form
\beq
f(\mu) &=&
\sum_{j=1}^{\tilde n} \bigg(\sum_{i:\ x_i=\tilde x_j} \frac {\mu(\Omega)}{n}\bigg) \delta_{y_i}
\nonumber\\
&=&
\sum_{j=1}^{\tilde n}  \frac {cm_j}{n} \delta_{\tilde y_j}
\eeq
where $c=\mu(\Omega)$ and the set $\{\tilde y_1,\dots,\tilde y_{\tilde n}\}$ contains the same points as  the set $\{y_1,\dots,y_n\}$. Below, we denote $\tilde X=(\tilde x_1,\dots,\tilde x_{\tilde n})$ and  $Y_j(\tilde X,\mu):=\tilde y_j.$ The above shows that the claim is valid when the $a_j$ are of the form $a_j = cm_j$, $m_j\in \mathbb Z_+$.

We now consider general values, $a_i>0$, and points, $x_i\in \Omega$, $i=1,\dots,n$ and let $c=\sum_{i=1}^{n}a_i$. We let $N_k, m_{k,i}\in\mathbb Z_+$,  $i=1,\dots,n$, $k\in \mathbb Z_+$ be such that
\ba
\lim_{k\to \infty} \frac { m_{k,i}}{N_k}=a_i\quad \hbox{for all }i=1,2,\dots,n.
\ea
Also, we let $c_k=c/N_k$ and
\beq\label{mu k measure}
\mu_k = \sum_{i=1}^{n} c_{k}m_{k,i}\delta_{x_i}.
\eeq
We write $X=(x_1,\dots, x_{n})$. Then, as we have already shown that the claim is valid for measures $\mu_k$ of the form
\eqref{mu k measure}, we can write $f(\mu_k)$ as
\beq
f(\mu_k)= \sum_{i=1}^{n} c_{k}m_{k,i}\delta_{Y_i(X,\mu_k)}= \sum_{i=1}^{n} 
\frac { m_{k,i}}{N_k}\delta_{Y_i(X,\mu_k)}\in \mathcal M^+_{fin,(n)}(\Omega).
\eeq
As $f$ is a continuous map in the 1-Wasserstein topology and  the  set $\mathcal M^+_{fin,(n)}(\Omega)$ is a closed subset of $\mathcal M^+(\Omega)$
in the same topology and $\mathcal M^+(\Omega)$ is a complete space,
we conclude that there exists a limit
\beq
f(\mu)=\lim_{k\to \infty}f(\mu_k) \in  \mathcal M^+_{fin,(n)}(\Omega).
\eeq
Thus we can write $f(\mu)$ in the form,
\beq
f(\mu)= \sum_{j=1}^{n'} b_j\delta_{z_j},
\eeq
with some  $n'\le n$, $z_j \in \Omega$ and $b_j>0$. We choose
\ba
& &\rho=\min\{|z_{j=1}-z_{j_2}|:\  j_1,j_2\in [n'],\ j_1\not =j_2\}>0
%,\\
%& &0<h<\min\{|a_{j=1}-a_{j_2}|:\ j_1\not =j_2\}
\ea
and let $A=\min_j a_j>0.$ Moreover, as $f(\mu_k)\to f(\mu)$ in the 1-Wasserstein metric as $k\to \infty$,  we observe that for each $k$ there is a partition of the set $\{1,2,\dots,n\}$ to a union of disjoint sets, $I_{1,k},\dots,I_{n',k}$, such that when $k$ is sufficiently large,
\ba
%\lim_{k\to \infty}
 \sum_{j=1}^{n'} \frac {A}{4} \min_{i\in I_{j,k}}\hbox{dist}(Y_i(X,\mu_k),z_j)
+\frac {1}{4}\sum_{j=1}^{n'} \bigg|\bigg(\sum_{i\in  I_{j,k}}\frac { m_{k,i}}{N_k}\bigg)-b_j\bigg|
\le W_1(f(\mu_k),f(\mu)).
\ea
As $W_1(f(\mu_k),f(\mu))\to 0$ as $k\to \infty$, we find that by replacing  $\mu_{k}$  by its suitable subsequence, % $\mu_{k_\ell}$,
we can assume that the partition $I_{1,k},\dots,I_{n',k}$  is equal to a partition $I_{1},\dots,I_{n'}$ that is independent of $k$, and 
\beq
Y_i(X,\mu_k)\to z_i,\quad \hbox{as }k\to \infty.
\eeq
Moreover,
\beq
\sum_{i\in  I_{j}} a_i = \sum_{i\in  I_{j}}\lim_{k\to \infty}\frac { m_{k,i}}{N_k} = b_j
\eeq
for $j=1,2,\dots,n'$. Then $b_j = \sum_{i\in  I_{j}} a_i$, and
\beq
f(\mu)= \sum_{j=1}^{n'} \bigg(\sum_{i\in  I_{j}} a_i\bigg) \delta_{z_j}= \sum_{i=1}^{n} a_i\delta_{y_i},
\eeq
where $y_1,\dots,y_n$ is a sequence of the points $z_1,\dots,z_{n'}$ where each $z_j$ appears 
%$\# I_{j}$ 
$|I_j|$
times. This proves the claim for general weights $a_i>0$.
\end{proof}

\subsection{Proof of Lemma~\ref{lem: extension of derivative}}
\label{app:lem: extension of derivative}

\begin{proof}
Let $f_{G} :\ \mathcal{M}^+(\Omega)\to \mathcal{M}^+(\RR^{d'})$ be the map, $f_G(\mu)=G(\mu)_\# \mu$, with a continuous map $G:\mathcal{M}^+(\Omega) \times \Omega \to \R^{d'}$. We assume that $\psi_k, \psi \in C^1_0(\R^{d'})$ and $\mu_k, \mu \in \mathcal M_+(\Omega)$, $k=1,2,\dots$ are sequences with 
$$
\psi_k \hbox{ is constant in an open neighborhood of }\supp(f(\mu_k))
$$
and
$$
\lim_{k\to\infty}W_1(\mu_k,\mu)=0,\quad
\lim_{k\to\infty} \|\psi_k -\psi\|_{L^\infty(\R^{d'})}=0.
$$ 
Let 
$$
\mu^\e_{k,x}:=\mu_k + \e \delta_x.
$$
Then, by the simple computation,
$$
\bra f_{G}(\mu^\e_{k,x}), \psi\ket
=
\int_{\R^{d}} \psi_k(G(\mu^\e_{k,x},y))d \mu_k(y) + \e \psi_k(G(\mu^\e_{k,x},x)).
$$
As the set $\Omega \subset \R^{d}$ is compact, 
$$
\mathcal M^+_\rho(\Omega) :=\{\mu \in \mathcal M^+(\Omega):\ \mu(\Omega)\le \rho\}
$$ is also compact by the Prokhorov’s theorem. Then, the map $G:\mathcal M^+_\rho(\Omega)\times \Omega\to \RR^{d'}$
is uniformly continuous. As $G(\mu^\e_{k,x},\cdot)\to G(\mu_k,\cdot)$ uniformly in $\Omega \subset \RR^{d}$ as $\epsilon \to 0$,
we see that  
$$
\sup_{y \in \text{supp}(\mu_k)} |G(\mu^\e_{k,x},y)-G(\mu_{k},y)|\to 0\quad\hbox{as }\epsilon\to 0.
$$
Thus, we find that for sufficiently small $ \epsilon \in (0, 1)$
$$
\psi_k(G(\mu^\e_{k,x},y))=\psi_k(G(\mu_k,y))
$$
for all $y \in \supp(\mu_k)$, and
$$
\bra f_{G}(\mu^\e_{k,x}), \psi_k\ket
= \int_{\R^d} \psi_k(G(\mu_k,y))d \mu_k(y)+\e \psi_k(G(\mu_k,x)).
\nonumber
$$
This implies that  
$$
\overline {\mathcal D}_{f_G}(\mu_k,x,\psi_k) = \lim_{\e \to + 0} \frac{\bra f_{G}(\mu^\e_{k,x}), \psi_k\ket - \bra f_{G}(\mu_{k}), \psi_k\ket }{\epsilon}
= \psi_k(G(\mu_k,y)).
$$
Upon taking the limit $k \to \infty$, we obtain
$$
\overline {\mathcal D}_{f_G}(\mu,x,\psi)
=\psi(G(\mu,y)).
$$
Next, we prove that the map $(\mu,y,\psi)\to \overline {\mathcal D}_{f_G}(\mu,y,\psi)$ is uniformly continuous.
Let $\e_1>0$. By the uniform continuity of $G$, there is a $\delta_1=\delta_1(\e_1)\in (0,\e_1)$ such that if $W_1(\mu_1,\mu_2)<\delta_1(\e_1)$ and $|y_1-y_2|<\delta_1(\e_1)$
then $|G(\mu_1,y_1)-G(\mu_2,y_2)|<\e_1/2$. Let $(\mu_1,y_1,\psi_1),(\mu_2,y_2,\psi_2) \in \mathcal X$ so that $\hbox{Lip}(\psi_j)\le \eta$ for $j=1,2$. Also, assume that $\|\psi_1-\psi_2\|_{L^\infty}<\delta_1(\e_1)$.
We then see that
\beq 
\hspace*{-0.75cm}
|\overline {\mathcal D}_{f_G}(\mu_1,y_1,\psi_1)-\overline {\mathcal D}_{f_G}(\mu_2,y_2,\psi_2)|
&=&|\psi_1(G(\mu_1,y_1))-\psi_2(G(\mu_2,y_2))|
\nonumber\\
&\le &|\psi_1(G(\mu_1,y_1))-\psi_1(G(\mu_2,y_2))|
\nonumber\\
& &\qquad +|\psi_1(G(\mu_2,y_2))-\psi_2(G(\mu_2,y_2))|
\nonumber\\
&\le &\hbox{Lip}(\psi_1)|G(\mu_1,y_1)-G(\mu_2,y_2)|
+\|\psi_1-\psi_2\|_{L^\infty}
\nonumber\\
&\le &\hbox{Lip}(\psi_1)\e_1+\delta_1(\e_1)
\nonumber\\
&\le &(\eta+1)\e_1.
\nonumber
\eeq
We observe that if $D_{\mathcal X}((\mu_1,y_1,\psi_1),(\mu_2,y_2,\psi_2))<\delta_1(\e_1)$ then $W_1(\mu_1,\mu_2)<\delta_1(\e_1)$ and $|y_1-y_2|<\delta_1(\e_1)$, and moreover that $\|\psi_1-\psi_2\|_{L^\infty}<\delta_1(\e_1)$.
% Thus the above implies that $|\overline {\mathcal D}_{f_G}(\mu_1,y_1,\psi_1)-\overline {\mathcal D}_{f_G}(\mu_2,y_2,\psi_2)|<(\eta+1)\e_1.$
We conclude that $(\mu,y,\psi) \to \overline {\mathcal D}_{f_G}(\mu,y,\psi)$ is uniformly continuous.
\end{proof}

\subsection{Proof of Lemma~\ref{Lemma yk continuous}}
\label{app-lemma yk}

\begin{proof}
The assumptions imply that
$W_1(\mu_p,\mu_0) \to 0$ as $p \to \infty$. Hence, as $f$ is continuous in the 1-Wasserstein distance, it holds that $W_1(f(\mu_p),f(\mu_0)) \to 0$ as $p \to \infty$.
If the claim is not valid, there are $k$ and $({\boldsymbol x}^p,{\boldsymbol a}^p)$ such that $({\boldsymbol x}^p,{\boldsymbol a}^p) \to ({\boldsymbol x}^0,{\boldsymbol a}^0)$ as $p \to \infty$ and  $\mu_0 = \sum_{i=1}^{n} a^0_i \delta_{x^0_i} \in \mathcal M_{fin,dif,(n)}(\Omega)$, and the sequence
$y_k({\boldsymbol x}^p;{\boldsymbol a}^p)$, $p\in {\mathbb Z}_+$, does not converge to
the value $y_k({\boldsymbol x}^0;{\boldsymbol a}^0)$ as $p\to \infty$.
%
%
% \footnote{(Matti, Sep. 2) Earlier, in the 1-dimensional case we had the following formulas:
% \beq\label{limsup formula 1 pre}
% \limsup_{p \to \infty}y_k({\boldsymbol x}^p;{\boldsymbol a}^p) \not =y_k({\boldsymbol x}^0;{\boldsymbol a}^0),
%  \eeq
%  or
%  \beq\label{liminf formula 1 pre}
% \liminf_{p \to \infty}y_k({\boldsymbol x}^p;{\boldsymbol a}^p)\not =y_k({\boldsymbol x}^0;{\boldsymbol a}^0).
%  \eeq}
By replacing $({\boldsymbol x}^p;{\boldsymbol a}^p)$ 
  by its suitable subsequence, we can assume that there exists 
  $z\in \R^{d'}$ such that
\beq\label{lim formula 1}
\lim_{p \to \infty}y_k({\boldsymbol x}^p;{\boldsymbol a}^p) =z\not =y_k({\boldsymbol x}^0;{\boldsymbol a}^0).
\eeq
% Below in this proof, we consider this fixed index  $k$ and the pairs $({\boldsymbol x}^p,{\boldsymbol a}^p)$ and $({\boldsymbol x}^0,{\boldsymbol a}^0)$.
As all $a_i^0$ are strictly positive and $a_i^p \to a_i^0$, there are $b>0$ and $p_0$ such that we have $a_i^p>b$ for all $p>p_0$ and $i$. As $f(\mu_p)=\sum_{k=1}^n a^p_k\delta_{y_k({\boldsymbol x}^p;{\boldsymbol a}^p)} \to f(\mu_0)$ in 1-Wasserstein distance, we see  that
$$
\lim_{p \to \infty} \sup_{y \in \text{supp}(f(\mu_p))}\hbox{dist}(y,\supp(f(\mu_0))=0.
$$
This and \eqref{lim formula 1} imply that there is $k_0\not= k$ such that 
  \beq\label{lim formula 2}
\lim_{p \to \infty}y_k({\boldsymbol x}^p;{\boldsymbol a}^p)=y_{k_0}({\boldsymbol x}^0;{\boldsymbol a}^0) \not =y_k({\boldsymbol x}^0;{\boldsymbol a}^0).
\eeq
Then, as \eqref{lim formula 2} holds,
%\eqref{limsup formula 1} or \eqref{liminf formula 1} holds, 
we find that
$$
\lim_{p \to \infty}
%\liminf_{p \to \infty}
|y_k({\boldsymbol x}^p;{\boldsymbol a}^p) -y_k({\boldsymbol x}^0;{\boldsymbol a}^0)| \ge  
\min\{ |y-y'|:\ y,y'\in \supp(f(\mu_0)),\ y \not =y'\} >0
$$
and that the measures $\mu_p = \sum_{i=1}^{n} a^p_i \delta_{x^p_i}$ and $\mu_0= \sum_{i=1}^{n} a^0_i \delta_{x^0_i}$ and their images under $f$, that is,
$$
f(\mu_p) = \sum_{i=1}^{n} a^p_i \delta_{y_i({\boldsymbol x}^p;{\boldsymbol a}^p)}\quad\hbox{and}\quad 
f(\mu_0)= \sum_{i=1}^{n} a^0_i \delta_{y_i({\boldsymbol x}^0;{\boldsymbol a}^0)},
$$
satisfy the inequality
\ba
\lim_{p \to \infty} 
W_1(f(\mu_p),f(\mu_0))
%&\ge& \bigg|\sum_{i\in \{j:\ y_j({\boldsymbol x}^p;{\boldsymbol a}^p)=
%} \bigg|
%
%\\
\ge \hbox{gap}(\mu_0)\min\{ |y-y'|:\ y,y'\in \supp(f(\mu_0)),\ y\not =y'\}>0,
\ea
where $\hbox{gap}(\mu_0)$ is defined in \eqref{gap finite with difference}. This is not possible in view of the 1-Wasserstein continuity of $f$. Thus, the claim follows.
\end{proof}

\subsection{Proof of Corollary~\ref{cor:approximation-trans}}
\label{app:approximation-trans}
\begin{proof}
Using Theorem~\ref{main thm}, there is an in-context map $G$ such that $f(\mu) = G(\mu)_{\sharp}\mu$. Since the map $G$ is continuous, by using \cite[Theorem 1]{furuya2024transformers}, for any $\epsilon \in (0,1)$, there is a measure-theoretic transformer-style in-context mapping $G_{\mathrm{tran}} := F_{\xi_L} \diamond \Gamma_{\theta_L} \diamond \ldots \diamond F_{\xi_1} \diamond \Gamma_{\theta_1}$ such that 
\[
\sup_{(\mu,x) \in \Pp(\Omega) \times \Omega} |G_{\mathrm{tran}}(\mu, x) - G(\mu, x)|\leq \epsilon,
\]
which implies that, by the duality theorem of Kantorovich and Rubinstein, 
\begin{align*}
 W_1(f_{\mathrm{tran}}(\mu), f(\mu)) 
 &
 =\sup_{\mathrm{Lip}(\varphi)\leq 1} \int \varphi(G_{\mathrm{tran}}(\mu,x)) - \varphi(G(\mu,x)) d\mu(x)
 \\
 &
 \leq \int |G_{\mathrm{tran}}(\mu,x) - G(\mu,x)| d\mu(x) \leq \epsilon.
\end{align*}
We have proved Corollary~\ref{cor:approximation-trans}.
\end{proof}

\subsection{Proof of Proposition~\ref{prop-f-T-hold-B}}
\label{app-prop-f-T-hold-B}

\begin{proof}
By \cite[Theorem 2.3]{piccoli2015control}, there exists $G_t : \Pp(\Omega) \times \Omega \to \R^{d}$ such that  
\[
\mu_t = G_t(\mu_0)_\sharp \mu_0, 
\]
where $G_t$ is defined by the unique solution of the following Cauchy problem
\begin{equation}
\label{eq:G-flow}
\partial_t G_t(\mu_0, x) = V[\mu(t)](t, G_t(\mu_0, x)), \quad G_0(\mu_0, x) = x.
\end{equation}
This is a push forward, thus the solution map, $f$, satisfies (B1). Moreover, if the map $(\mu, x) \mapsto G_T(\mu, x)$ is Lipschitz continuous, by Lemma~\ref{lem: extension of derivative} the map $(\mu,x,\psi) \mapsto \overline {\mathcal D}_{f_T}(\mu,x,\psi)$ is Lipschitz continuous with respect to the metric $D_{\mathcal X}$. This implies (B2). In what follows, we will prove that the map $(\mu, x) \mapsto G_T(\mu, x)$ is Lipschitz continuous. 

\bigskip

We estimate, for $\mu_0, \nu_0 \in \Pp(\Omega)$ and $x,y \in \Omega$,
\begin{align*}
\frac{d}{dt} \| &G_t(\mu_0, x) - G_t(\nu_0, y)  \|_2
\\
&
\leq 
\Big\| \frac{d}{dt} G_t(\mu_0, x) - \frac{d}{dt} G_t(\nu_0, y)  \Big\|_2
= 
\| V[\mu(t)](t, G_t(\mu_0, x)) - V[\nu(t)](t, G_t(\nu_0, x))  \|_2
\\
&
\leq 
\| V[\mu(t)](t, G_t(\mu_0, x)) - V[\mu(t)](t, G_t(\nu_0, x))  \|_2
\\
&
\quad\quad\quad\quad
+
\| V[\mu(t)](t, G_t(\nu_0, x)) - V[\nu(t)](t, G_t(\nu_0, x))  \|_2
\\
&
\leq 
L(t) \| G_t(\mu_0, x) - G_t(\nu_0, x) \|_2
+
K(t) W_1(\mu(t), \nu(t))
\\
&
\leq 
L(t) \| G_t(\mu_0, x) - G_t(\nu_0, x) \|_2
+
K(t) e^{C_T t} W_1(\mu_0, \nu_0),
\end{align*}
for some $K, L \in L^{\infty}_{loc}(\R)$ and some $C_T > 0$. Here, we have used the assumption of Lipschitz continuity required in \cite[Theorem 2.3]{piccoli2015control}, and the stability estimate \cite[(2.3)]{piccoli2015control}. By Gronwall's inequality, we find that 
\begin{align*}
&
\| G_t(\mu_0, x) - G_t(\nu_0, x) \|_2
\\
&
\leq e^{A(t)} \underbrace{\|G_0(\mu_0, x) - G_0(\nu_0, x)\|_2}_{=\|x-y\|_2} 
+ \left( \int_0^t e^{A(t) - A(s)} K(s) e^{C_T s} ds\right)  W_1(\mu_0, \nu_0),
\end{align*}
where $A(t) = \int_0^{t} L(s) ds$. Thus, substituting $t=T$, there exists $C'_T >0$ such that 
\[
\| G_T(\mu_0, x) - G_T(\nu_0, x) \|_2
\leq C'_T (W_1(\mu_0, \nu_0) + \|x-y\|_2 ),
\]
which implies that the map $(\mu, x) \mapsto G_T(\mu, x)$ is Lipschitz continuous.
\end{proof}

\section{Proof of Theorem \ref{main thm}}

\subsection{Part 1 : (A1)-(A2) imply (B1)-(B2)}
\label{app:sec:proof-1}

Assume that (A1) and (A2) hold true. Then, let $f_{G} :\ \mathcal{M}^+(\Omega)\to \mathcal{M}^+(\RR^{d'})$ be the map, $f_G(\mu)=G(\mu)_\# \mu$, with $G:\mathcal{M}^+(\Omega) \times \Omega \to \R^{d'}$. It is straightforward to prove (B1) and, hence, we will focus on proving that (B2) holds.
We assume that $\psi_k, \psi \in C^1_0(\R^{d'})$ and $\mu_k, \mu \in \mathcal M_+(\Omega)$, $k=1,2,\dots$ are sequences with 
$$
\psi_k \hbox{ is constant in an open neighborhood of }\supp(f(\mu_k))
$$
and
$$
\lim_{k\to\infty}W_1(\mu_k,\mu)=0,\quad
\lim_{k\to\infty} \|\psi_k-\psi\|_{L^\infty(\R^{d'})}=0.
$$ 
Let 
$$
\mu^\e_{k,x}:=\mu_k + \e \delta_x.
$$
Then 
\beq
\bra f_{G}(\mu^\e_{k,x}), \psi\ket
&=& \bra (G(\mu^\e_{k,x}))_\# \mu^\e_{k,x}, \psi_k\ket
\nonumber\\[0.25cm]
&=& \bra \mu^\e_{k,x}, \psi_k \circ G(\mu^\e_{k,x},\cdot)\ket
\nonumber\\
&=&  \int_{\R^{d}} \psi_k(G(\mu^\e_{k,x},y))d \mu^\e_{k,x}(y)
\nonumber\\
&=& 
\int_{\R^{d}} \psi_k(G(\mu^\e_{k,x},y))d \mu_k(y) + \e \psi_k(G(\mu^\e_{k,x},x)).
\nonumber
\eeq
By (A2), 
$
(\mu,x)\to G(\mu,x)
$ is a continuous function. 
As the set $\Omega\subset \RR^{d}$
is compact, for any $\rho>0$ the set 
$$
\mathcal M^+_\rho(\Omega) :=\{\mu \in \mathcal M^+(\Omega):\ \mu(\Omega)\le \rho\},
$$
consists of measures that are uniformly bounded and supported
in the same compact set $\Omega$. Therefore, the set $\mathcal M^+_\rho(\Omega)$
is tight (see \cite[Chapter 1, Section 1]{Billingsley}). 
The set $\mathcal M^+_\rho(\Omega)$ is also closed in the weak${}^*$ topology
of measures as it is closed in the 1-Wasserstein topology.
By  Prokhorov's theorem (see \cite[Chapter 5, Theorem 5.1]{Billingsley}), 
the set  $\mathcal M^+_\rho(\Omega)$ is compact in the 1-Wasserstein topology.
As a continuous map defined in a compact metric space is uniformly
continuous, the map $G:\mathcal M^+_\rho(\Omega)\times \Omega\to \RR^{d'}$
is uniformly continuous.
Moreover, by our assumptions, the derivative of $\psi_k$ is zero in some neighborhood, $V \subset \RR^{d'}$, of the finite set $\{G(\mu_k,x):\ x\in \supp(\mu_k)\}$. As $G(\mu^\e_{k,x},\cdot)\to G(\mu_k,\cdot)$ uniformly in $\Omega \subset \RR^{d}$ as $\epsilon \to 0$,
we see that  
\beq
\sup_{y \in \text{supp}(\mu_k)} |G(\mu^\e_{k,x},y)-G(\mu_k,y)|\to 0\quad\hbox{as }\epsilon\to 0.
\eeq
Thus, we find that for sufficiently small $ \epsilon \in (0, 1)$, for all $y \in \supp(\mu_k)$ the point $G(\mu^\e_{k,x},y)$ belongs to the set $V$, and, hence, 
$$
\psi_k(G(\mu^\e_{k,x},y))=\psi_k(G(\mu_k,y))
$$
for all $y \in \supp(\mu_k)$, and
\beq
\bra f_{G}(\mu^\e_{k,x}), \psi_k\ket
&=& \int_{\R^d} \psi_k(G(\mu_k,y))d \mu_k(y)+\e \psi_k(G(\mu_k,x)).
\nonumber
\eeq
This implies that  
\beq\label{step 1}
\frac d{d\e}\bigg|_{\e=0} \bra f_{G}(\mu^\e_{k,x}), \psi_k\ket
=0 + \psi_k(G(\mu_k,y))=\psi_k(G(\mu_k,y)).
\eeq
Thus,
$$
\overline {\mathcal D}_{f_G}(\mu_k,x,\psi_k)
=\psi_k(G(\mu_k,x)).
$$
We see that 
$$
\lim_{k \to\infty} |\psi_k(G(\mu_k,x))-\psi_k(G(\mu_k,x))|
\leq \lim_{n \to\infty} \|\psi_k -\psi\|_{L^\infty(\R^{d'})}=0,
$$
and as $\mu \to G(\mu,x)$ is continuous, we have
$$
\lim_{k\to\infty} \psi(G(\mu_k,x))=\psi(G(\mu,x)).
$$
But then
$$
\lim_{k\to\infty} \psi_k(G(\mu_k,x))=\psi(G(\mu,x)).
$$
As the above holds for all $\mu_k$ and $\psi_k$ that converge to $\mu$ and $\psi$ in the way stated in  Definition~\ref{def:support-preserving-maps}, we conclude that the regular part of the derivative $\overline {\mathcal D}_{f_G}(\mu,y,\psi)$ exists and is equal to
\beq
\overline {\mathcal D}_{f_G}(\mu,x,\psi)
=\psi(G(\mu,y)).
\label{psi G formula}
\eeq
This proves the existence of the regular part of the derivative. Moreover, these arguments prove Lemma~\ref{lem: extension of derivative}.

Next, we prove that the map $(\mu,y,\psi)\to \overline {\mathcal D}_{f_G}(\mu,y,\psi)$ is uniformly continuous when (A1) and (A2) are valid.
To this end, let $\e_1>0$. By (A2), there is a $\delta_1=\delta_1(\e_1)\in (0,\e_1)$ such that if $W_1(\mu_1,\mu_2)<\delta_1(\e_1)$ and $|y_1-y_2|<\delta_1(\e_1)$
then $|G(\mu_1,y_1)-G(\mu_2,y_2)|<\e_1/2$. Let $(\mu_1,y_1,\psi_1),(\mu_2,y_2,\psi_2) \in \mathcal X$ so that $\hbox{Lip}(\psi_j)\le \eta$ for $j=1,2$. Also, assume that $\|\psi_1-\psi_2\|_{L^\infty}<\delta_1(\e_1)$.
Equation \eqref{psi G formula} implies that
\beq 
\hspace*{-0.75cm}
|\overline {\mathcal D}_{f_G}(\mu_1,y_1,\psi_1)-\overline {\mathcal D}_{f_G}(\mu_2,y_2,\psi_2)|
&=&|\psi_1(G(\mu_1,y_1))-\psi_2(G(\mu_2,y_2))|
\nonumber\\
&\le &|\psi_1(G(\mu_1,y_1))-\psi_1(G(\mu_2,y_2))|
\nonumber\\
& &\qquad +|\psi_1(G(\mu_2,y_2))-\psi_2(G(\mu_2,y_2))|
\nonumber\\
&\le &\hbox{Lip}(\psi_1)|G(\mu_1,y_1)-G(\mu_2,y_2)|
+\|\psi_1-\psi_2\|_{L^\infty}
\nonumber\\
&\le &\hbox{Lip}(\psi_1)\e_1+\delta_1(\e_1)
\nonumber\\
&\le &(\eta+1)\e_1.
\nonumber
\eeq
We observe that if $D_{\mathcal X}((\mu_1,y_1,\psi_1),(\mu_2,y_2,\psi_2))<\delta_1(\e_1)$ then $W_1(\mu_1,\mu_2)<\delta_1(\e_1)$ and $|y_1-y_2|<\delta_1(\e_1)$, and moreover that $\|\psi_1-\psi_2\|_{L^\infty}<\delta_1(\e_1)$.
% Thus the above implies that $|\overline {\mathcal D}_{f_G}(\mu_1,y_1,\psi_1)-\overline {\mathcal D}_{f_G}(\mu_2,y_2,\psi_2)|<(\eta+1)\e_1.$
We conclude that $(\mu,y,\psi) \to \overline {\mathcal D}_{f_G}(\mu,y,\psi)$ is uniformly continuous. This proves (B2).

We continue with proving one direction of the final statement of the theorem. Let $G(\mu,y)$ be a Lipschitz map. Equation \eqref{psi G formula} implies that for all $(\mu_1,y_1,\psi_1), (\mu_2,y_2,\psi_2) \in \mathcal X$, 
\beq 
|\overline {\mathcal D}_{f_G}(\mu_1,y_1,\psi_1)-\overline {\mathcal D}_{f_G}(\mu_2,y_2,\psi_2)|
&=&|\psi_1(G(\mu_1,y_1))-\psi_2(G(\mu_2,y_2))|
\nonumber\\
&\le &|\psi_1(G(\mu_1,y_1))-\psi_1(G(\mu_2,y_2))|
\nonumber\\
& &\qquad +|\psi_1(G(\mu_2,y_2))-\psi_2(G(\mu_2,y_2))|
\nonumber\\
&\le &\hbox{Lip}(\psi_1)|G(\mu_1,y_1)-G(\mu_2,y_2)|
+\|\psi_1-\psi_2\|_{L^\infty}
\nonumber\\
&\le &(\hbox{Lip}(\psi_1)\hbox{Lip}(G)+1)D_{\mathcal X}((\mu_1,y_1,\psi_1),(\mu_2,y_2,\psi_2))
\nonumber\\
&\le &(\eta \hbox{Lip}(G)+1)D_{\mathcal X}((\mu_1,y_1,\psi_1),(\mu_2,y_2,\psi_2)).
\nonumber
\eeq
Hence, $(\mu,y,\psi)\to \overline {\mathcal D}_{f_G}(\mu,y,\psi)$ is a Lipschitz map.

%Formula \eqref{psi G formula} implies that, if $(\mu,x) \to G(\mu,x)$ is a Lipschitz map, we see that for $(\mu_1,y_1,\psi_1),(\mu_2,y_2,\psi_2) \in \mathcal X$, 
%\beq 
%|\overline {\mathcal D}_{f_G}(\mu_1,y_1,\psi_1)-\overline {\mathcal D}_{f_G}(\mu_2,y_2,\psi_2)|
%&=&|\psi_1(G(\mu_1,y_1))-\psi_2(G(\mu_2,y_2))|
%\nonumber\\
%&\le &|\psi_1(G(\mu_1,y_1))-\psi_1(G(\mu_2,y_2))|
%\nonumber\\
%& &\qquad +|\psi_1(G(\mu_2,y_2))-\psi_2(G(\mu_2,y_2))|
%\nonumber\\
%&\le &\hbox{Lip}(\psi_1)|G(\mu_1,y_1)-G(\mu_2,y_2)|
%+\|\psi_1-\psi_2\|_{L^\infty}
%\nonumber\\
%&\le &(\hbox{Lip}(\psi_1)\hbox{Lip}(G)+1)D_{\mathcal X}((\mu_1,y_1,\psi_1),(\mu_2,y_2,\psi_2))
%\nonumber\\
%&\le &(\eta \hbox{Lip}(G)+1)D_{\mathcal X}((\mu_1,y_1,\psi_1),(\mu_2,y_2,\psi_2)),
%\nonumber
%\eeq
%This proves (B2).

\subsection{Part 2 : (B1)-(B2) imply (A1)-(A2)}
\label{app:sec:proof-2}

Assume that (B1) and (B2) hold true. Since $f$ is a support-preserving map, $f :\ \mathcal M^+(\Omega)\to \mathcal M^+(\RR^{d'})$, there are (possibly non-continuous) functions,
$$
    y_i:\Omega^n \times (0,\infty)^n \to \RR^{d'} ,\
    ({\boldsymbol{x}},{\boldsymbol a}) \to y_i({\boldsymbol x};{\boldsymbol a}),\quad
    i=1,2,\dots,n,
$$
where
$$
\boldsymbol{x} = (x_1, \dots, x_n)
\quad \text{and} \quad
\boldsymbol{a} = (a_1, \dots, a_n),
$$
such that the following holds: Let
$$
    \mu = \sum_{i=1}^{n} a_i \delta_{x_i} \in \mathcal M_{fin}(\Omega),\quad a_i>0;
$$
then the functions $y_i({\boldsymbol x};{\boldsymbol a})$ satisfy
\[
    f(\mu) = \sum_{i=1}^{n} a_i \delta_{y_i({\boldsymbol x};{\boldsymbol a})}.
\]
When $\mu \in \mathcal M^+_{fin,dif,(n)}(\Omega)$ (which is a refinement of the property that if $j\not =i$ then $a_j\not =a_i$),
% we see that by the definition of $\mathcal M^+_{fin,dif,(n)}(\Omega)$, that
the functions $({\boldsymbol x};{\boldsymbol a}) \to y_i({\boldsymbol x};{\boldsymbol a})$ must have the following property,
\beq \label{same xs}
\hbox{if $x_j=x_i$ then }y_j({\boldsymbol x};{\boldsymbol a}) = y_i({\boldsymbol x};{\boldsymbol a}).
\eeq 

Let $\mu \in \mathcal M^+(\Omega)$ and $x \in \Omega$, and $\alpha \in C^\infty_0(\R^d)$ be a cutoff function
such that  $\alpha(x)=1$ for all  $x \in \Omega$ and $\hbox{Lip}(\alpha(x)\cdot x)\le \eta$. We define
\begin{equation}
\label{eq:def-G}
G(\mu,x):= 
\begin{pmatrix}
\overline {\mathcal D}_f(\mu, x, \alpha \pi_1)
\\
\vdots
\\
\overline {\mathcal D}_f(\mu, x, \alpha  \pi_{d'})
\end{pmatrix},
\end{equation}
where $\pi_\ell :\mathbb{R}^d \to \mathbb{R}$ is the projection $\pi_\ell(x)=x_\ell$ onto the $\ell$-th component. By (B2), the map $(\mu, x) \mapsto G(\mu,x)$ is continuous, which proves (A2). In what follows, we will prove (A1).

\paragraph{The case when $\mu \in \mathcal M^+_{fin,dif,(n)}(\Omega)$.} We let
$$
\mu=\sum_{i=1}^{n}a_i \delta_{x_i} \in \mathcal M^+_{fin,dif,(n)}(\Omega)
$$
and
$$
f(\mu) =  \sum_{i=1}^{n} a_i \delta_{y_i({\boldsymbol x};{\boldsymbol a})}. 
$$
We define the measures, 
$$
\mu^\epsilon_x = \epsilon\delta_x + \sum_{i=1}^{n} a_i \delta_{x_i}.
$$
We observe that when $\epsilon>0$ is small enough, it holds that $\mu^\epsilon_x  \in \mathcal M^+_{fin,dif,(n)}(\Omega)$ if $x\in \{x_1,\dots,x_n\}$, or  $\mu^\epsilon_x  \in \mathcal M^+_{fin,dif,(n+1)}(\Omega)$ if $x \not\in \{x_1,\dots,x_n\}$.

\medskip

With the notation, $({\boldsymbol x},x) = (x_1,\dots,x_n,x)$, $({\boldsymbol a},\epsilon) = (a_1,\dots,a_n,\epsilon)$ and sometimes indicating the number, $n$ say, of variables in the function $y_i$ as $y^{(n)}_i$, we find that
\beq\label{f mu eps}
    f(\mu^\epsilon_x) = \epsilon \delta_{y^{(n+1)}_{n+1}({\boldsymbol x},x;{\boldsymbol a},\epsilon)} + \sum_{i=1}^{n} a_i \delta_{y^{(n+1)}_i({\boldsymbol x},x;{\boldsymbol a},\epsilon)}.
\eeq
We consider the case when $x=x_j$. Then,   
\beq\label{recombination1}
\bigg(\sum_{i=1}^n a_i\delta_{x_i}+\epsilon \delta_x\bigg)\bigg|_{x=x_j}
= \sum_{i\in \{1,\dots,n\}\setminus \{j\} } a_i \delta_{x_i}
+ (a_j+\epsilon ) \delta_{x_j} \in \mathcal M^+_{fin,(n)}(\RR).
\eeq
Thus, when we write $x=x_{n+1}=x_j$, it holds that
\beq\label{recombination2}
y_{n+1}^{(n+1)}({\boldsymbol x},x;{\boldsymbol a},\epsilon)\bigg|_{x=x_j}=
y_{n+1}^{(n+1)}({\boldsymbol x},x_{n+1};{\boldsymbol a},\epsilon)\bigg|_{x_{n+1}=x_j}=y_{j}^{(n)}({\boldsymbol x};{\boldsymbol a}+\epsilon e_j),
\eeq
where $e_j=(0,0,\dots,0,1,0,\dots,0)=(\delta_{ij})_{i=1}^n$, whence
$$
{\boldsymbol a}+\epsilon e_j=(a_1,\dots,a_{j-1},a_j+\e,a_{j+1},\dots,a_n).
$$
By Lemma~\ref{Lemma yk continuous}, and using equations \eqref{recombination1}-\eqref{recombination2}, we arrive at
\beq
\lim_{\epsilon \to 0+} y_{n+1}^{(n+1)}({\boldsymbol x},x;{\boldsymbol a},\epsilon) \bigg|_{x=x_j} = y_j^{(n)}({\boldsymbol x};{\boldsymbol a}).
\label{recombination3}
\eeq
Let $\ell \in  [d']$. We choose $\psi^{(\ell)}_k \in C^1_0(\mathbb{R}^{d'})$ such that
$$
\psi^{(\ell)}_k\quad \hbox{is constant in an open neighborhood of }\supp(f(\mu))
$$
and
$$
\hbox{Lip}(\psi^{(\ell)}_k)\le \eta \text{ together with }
 \lim_{n\to\infty} \|\psi^{(\ell)}_k -\alpha\pi_{\ell} \|_{L^\infty(\R^{d'})}=0.
$$
Thus, by using equation \eqref{recombination3}, we obtain
\beq \nonumber
\overline {\mathcal D}_f(\mu,x,\psi^{(\ell)}_k)\bigg|_{x=x_j} &=&
\overline {\mathcal D}_f(\mu,x_{n+1},\psi^{(\ell)}_k)\bigg|_{x_{n+1}=x_j}
\\ \nonumber
&=& \lim_{\epsilon \to 0+} \bra \psi^{(\ell)}_k,\delta_{y^{(n+1)}_{n+1}({\boldsymbol x},x_{n+1};{\boldsymbol a},\epsilon)} \ket \bigg|_{x_{n+1}=x_j}
\\ \nonumber
&=& \lim_{\epsilon \to 0+} \psi^{(\ell)}_k(y^{(n+1)}_{n+1}({\boldsymbol x},x_{n+1};{\boldsymbol a},\epsilon))\bigg|_{x_{n+1}=x_j}
\\ \nonumber
&=& \lim_{\epsilon \to 0+}\psi^{(\ell)}_k(y^{(n)}_{j}({\boldsymbol x};{\boldsymbol a}+\epsilon e_j))
\\[0.25cm] \nonumber
&=& \psi^{(\ell)}_k(y_j^{(n)}({\boldsymbol x};{\boldsymbol a}))
= \psi^{(\ell)}_k(y_j({\boldsymbol x};{\boldsymbol a})).
\eeq
By the definition of $\overline {\mathcal D}_f(\mu,x,\alpha\pi_\ell)$ and that $\psi^{(\ell)}_k\to \alpha\pi_\ell$ in $L^\infty(\R^{d'})$ as $n \to \infty$, we observe that for $x=x_j$, 
\ba
\pi_{\ell}(G(\mu,x_j))
&=&
\overline {\mathcal D}_f(\mu, x_j, \alpha\pi_{\ell})
= \lim_{n \to \infty }\overline {\mathcal D}_f(\mu, x_j, \psi^{(\ell)}_k)
\\
&=&
\lim_{n \to \infty} \psi^{(\ell)}_k(y_j({\boldsymbol x};{\boldsymbol a}))
= (\alpha \pi_{\ell})(y_j({\boldsymbol x};{\boldsymbol a})) = \pi_{\ell}(y_j).
\ea
This proves that, for each $j \in [n]$,
\beq\label{G-formula}
G(\mu, x_j) = y_j({\boldsymbol x};{\boldsymbol a}),
\eeq
which is equivalent to
$$
f(\mu)=(G_\mu)_\# \mu\quad \text{ for } \mu \in \mathcal M^+_{fin,dif,(n)}(\Omega).
$$

\paragraph{The case when $\mu \in \mathcal M^+(\Omega)$.}

Let $\mu$ be a (possibly not-discretely supported) measure $\mu \in \mathcal M^+(\Omega)$. By Lemma~\ref{lem:dense}, we choose the sequence $(\tilde \mu_k)_{k \in \N} \subset \mathcal M^+_{fin,dif}(\Omega)$ such that ${\tilde \mu_k} \to \mu$ as $k \to \infty$, where the limit is considered in the 1-Wasserstein topology. We have already shown that for $\tilde \mu_k \in {\mathcal M}^+_{fin,dif}(\Omega)$,
$$
f({\tilde \mu_k})=(G_{{\tilde \mu_k}})_\#({\tilde \mu_k}).
$$
Hence, taking the limit,
\beq
\label{Key identity for finite measures conseq}
f(\mu)=\lim_{m\to\infty}(G_{{\tilde \mu_k}})_\#({\tilde \mu_k}).
\eeq
That is, for all $\psi \in C^1_0(\R^{d'})$,
\beq
\label{Key identity for finite measures conseq2}
\bra \psi,f(\mu)\ket =\lim_{m\to\infty}\bra \psi,(G_{{\tilde \mu_k}})_\#({\tilde \mu_k})\ket,
\eeq
where
$$
\bra \psi,(G_{{\tilde \mu_k}})_\#({\tilde \mu_k})\ket=\bra \psi\circ G_{{\tilde \mu_k}},{\tilde \mu_k}\ket=\int_{\R^{d'}} 
\psi(G_{{\tilde \mu_k}}(x))\,d{{\tilde \mu_k}}(x).
$$
But then
\begin{multline}\label{Key identity for finite measures conseq5}
\lim_{k\to\infty}\bra \psi,(G_{{\tilde \mu_k}})_\#({\tilde \mu_k})\ket = \lim_{k\to\infty} \int_\R 
\psi(G_{{\tilde \mu_k}}(x))\,d{{\tilde \mu_k}}(x)
\\
= \lim_{k\to\infty}\int_\R 
(\psi(G({\tilde \mu_k},x))-\psi(G(\mu,x)))\,d{{\tilde \mu_k}}(x)
+\int_\R 
\psi(G(\mu,x))\,d{{\tilde \mu_k}}(x).
\end{multline}
By condition (B2), $(\mu,x) \to G(\mu,x)$ is uniformly continuous so that, using the compactness of $\Omega$,
$$
\|\psi(G({\tilde \mu_k},\cdotp))- \psi(G(\mu,\cdotp))\|_{L^\infty(\Omega)}
\le  \|\psi\|_{C^1} \|G({\tilde \mu_k},\cdotp)-G(\mu,\cdotp)\|_{L^\infty(\Omega)} \to 0 \text{ as } k \to \infty.
$$
Hence, equations \eqref{Key identity for finite measures conseq2} and \eqref{Key identity for finite measures conseq5} imply that
\beq\label{Key identity for finite measures conseq6}
\bra \psi,f(\mu)\ket 
&=&
0+ \lim_{k\to\infty}\int_\R 
\psi(G(\mu,x))\,d{{\tilde \mu_k}}(x)
\nonumber\\
&=&\lim_{k\to\infty}
\bra\psi(G(\mu,\cdotp)),{\tilde \mu_k}\ket
=\bra\psi(G(\mu,\cdotp)),\mu\ket
= \bra\psi,(G_\mu)_\#\mu\ket
\eeq
%$\psi=Id_{R_1}$, 
for all $\psi\in C^1_0(\R^{d'})$, $\hbox{Lip}(\psi)\le \eta$.
As both sides of \eqref{Key identity for finite measures conseq6} are linear in $\psi$, we see that  \eqref{Key identity for finite measures conseq6} holds for all $\psi\in C^1_0(\R^{d'})$ and, therefore, for all $\psi \in C_0(\R^{d'})$. Thus,
$$
f(\mu)=(G_\mu)_\# \mu\quad \text{ for } \mu \in \mathcal M^+(\Omega).
$$
This implies (A1). 

Finally, we observe that if $(\mu,y,\psi)\to \overline {\mathcal D}_{f_G}(\mu,y,\psi)$ is a Lipschitz map then $(\mu,y)\to G(\mu,y)
=(\overline {\mathcal D}_f(\mu, y, \alpha  \pi_j))_{j=1}^{d'}$ is also Lipschitz.

\section{The regular part of the derivative}
\label{app:remarks-reg}

We provide some perspectives on the regular part of the derivative introduced in the main text, in the following remarks.

\begin{remark}
A similar situation occurs when one defines the generalization of a derivative for a Lipschitz function $h :\ \RR^d \to \RR$. By the Rademacher theorem, the classical derivative of $h$ exists outside a zero-measurable set; to overcome this, one defines a weak derivative that is a function in $L^1_{loc}(\R^d)$ and is defined almost everywhere. We recall that the weak derivative is defined, in the sense of distributions, by the formula
$$
    \bra \partial_{x_i} h,\psi\ket=-\int_{\R^d} h(x)\partial_{x_i}\psi(x)dd,\quad\hbox{for }\psi\in C^\infty_0(\R^d).
$$
In the case when $h$ is a $C^1$-function, the classical derivative coincides with the weak derivative and the distributional duality coincides with the $L^2$-inner product
$$
    \bra \partial_{x_i} h,\psi\ket=\int_{\R^d} \partial_{x_i} h(x)\psi(x)dx.
$$
In this setting, the weak derivative is defined for a larger class of functions as a ``new'' generalized function.
    
Our definition of the regular part of the derivative is defined as a new generalized function using duality (or, in the weak sense). This definition is formally quite different from the classical one of Fr\'{e}chet derivative. However, as we see in Lemma~\ref{lem: extension of derivative}, for map $f_G$ defined with a smooth in-context function $G$, the regular part of derivative $\overline {\mathcal D}_{f_G}(\mu,x,\psi)$ coincides with the above defined object, $D^{reg}_{f_G}(\mu,x,\psi)$. So, we consider $D^{reg}_{f_G}(\mu,x,\psi)$ as a new object that is different from the classical Fr\'{e}chet derivative, and show that the definition of $D^{reg}_{f_G}(\mu,x,\psi)$ can be extended as a generalized regular part of the derivative, $\overline{\mathcal D}_{f}(\mu,x,\psi)$, for a class of functions $f$ for which we do not assume that the classical Fr\'{e}chet derivative is well-defined. 
\end{remark}

\begin{remark}
We point out that for any support preserving map $f :\ \mathcal{M}^+(\Omega)\to \mathcal{M}^+(\R^{d'})$, $\mu \in \mathcal{M}^+(\Omega)$, and  $\psi\in C^1_0(\RR^{d'})$, we can find a sequence of finitely supported measures, $\mu_k \in \mathcal{M}^+_{fin}(\Omega)$, that converges in the $1$-Wasserstein topology to $\mu$ as $n \to \infty$. Then also $\supp(f(\mu_k))$ is finitely supported, and we can denote $\supp(f(\mu_k))=\{y_{1,n}, y_{2,n}, \dots, y_{m_n,n}\}$. We can modify the function $\psi$ in a small neighborhood of each point, $y_{j,n}$, so that we obtain a function
$\psi_k \in  C^1_0(\RR^{d'})$ that satisfies \eqref{locally constant} and $\psi_k$ converges in the $L^\infty$ topology to $\psi$ as $n \to \infty$.  Thus, we see that for all measures $\mu \in \mathcal{M}^+(\Omega)$ and $\psi\in C^1_0(\RR^{d'})$, we can find sequences $\mu_k$ and $\psi_k$ that satisfy the conditions in Definition~\ref{genearlized regular derivative}. The existence of $\overline {\mathcal D}_{f}(\mu,x,\psi)$ thus means that for all $\mu$, $x$, and $\psi$, the limits \eqref{A limit} exist and are independent of the chosen sequences $\mu_k$ and $\psi_k$.
\end{remark}

\begin{remark}
We can come up with an alternative version of the definition for the regular part of the Fr\'{e}chet derivative of $f$ and of Definition~\ref{genearlized regular derivative}. Let us consider the triplets of measures $\mu$, points $x$ and test functions $\psi$ having the property that the test functions are locally constant in the support of $f(\mu)$. We denote this set by
\begin{multline}
\mathcal P_{lc} = \mathcal P_{lc}(f, \Omega, \rho,\eta) = \{(\mu,x,\psi) \in \mathcal M^+(\Omega) \times
\Omega \times C^1_0(\RR^{d'}) : 
\\
\mu(\Omega) \leq \rho, 
\ \  
\psi \text{ is constant in an open neighborhood of } f(\mu),
\ \   \|\psi\|_{C^1} \le \eta\}.
\nonumber
\end{multline}
Let
$$
\mathcal L_f(\mu,x,\psi) = \bra \psi,  D_\mu f(\mu)[\delta_x]) \ket,
$$
be the duality of the Fr\'{e}chet derivative $D_\mu f(\mu)[\delta_x]$ and the test function $\psi$. Then the restriction of $\mathcal L_f$ to the set $\mathcal P_{lc}$, that is,
$$
\mathcal L_f|_{\mathcal P_{lc}}:\mathcal P_{lc} \to \RR,
$$
coincides with the regular part of the derivative $\overline {\mathcal D}_{f}(\mu,x,\psi)$ of $f$. When the regular part of the derivative of $f$ exists, this map has a continuous extension to the set ${\mathcal X}$,
\[
\mathcal L_f^{ext} :\ \mathcal X \to \RR,
\]
in the topology determined by the metric, $D_{\mathcal X}$. This extension is the map $(\mu,x,\psi) \to \overline {\mathcal D}_{f}(\mu,x,\psi)$. Hence,
% the  $L^\infty$-regular part of the Fr\'{e}chet derivative of $f$, 
$\overline{\mathcal D}_f(\mu,x,\psi)$ given in Definition~\ref{genearlized regular derivative} can also be defined as the extension of the usual Fr\'{e}chet derivative from the set ${\mathcal P_{lc}}$ to the completion of this set in the appropriate topology.
\end{remark}

\section{MLPs with skip connections and composition forming an in-context map}
\label{app:transformers for discrete measures}

% To clarify the definitions and the relation of in-context map  to transformers operating to sequences of tokens, we write here formulas for  single-layer ``measure-theoretic'' transformers \cite{furuya2024transformers, castin2024smooth} based on multi-head self attention, for discrete measures.

%\subsection{Attention and operations for  discrete measures}

%Moreover, for the measure $\mu$ given in \eqref{mu measure1} it holds that
%\begin{align}
%    f_\Gamma \bigg(\sum_{i=1}^n\frac 1n \delta_{x_i}\bigg) &=\Gamma(\mu,\cdot)_\# \mu\\
%    &= \sum_{i=1}^n \frac 1n \delta_{y_i}
%   \end{align}
%where
%\begin{eqnarray} \label{ypoints}
%y_i=\Gamma(\mu,x_i)=x_i + 
%    \sum_{h=1}^H W^h
%    \sum_{\ell=1}^n \frac{
%        \exp\Big(
%            \frac{1}{\sqrt{k}}
%            (Q^h x_i)^\top (K^h x_\ell)
%        \Big)
%    }{
%        \sum_{j=1}^n 
%        \exp\Big(
%            \frac{1}{\sqrt{k}}
%            ((Q^h x_i)^\top (K^h x_j)
%        \Big)
%    } V^h x_\ell\, .
%\end{eqnarray}

%As discussed in Subsection \ref{subsec: in context maps}, 
%the measure-to-measure map $f_\Gamma$ defines
%a map $$F_\Gamma=\iota^{-1}\circ f_\Gamma\circ \iota$$ that 
%maps a sequence of tokens, $x_1,x_2,\dots,x_n\in \Omega\subset \R^d$ to another sequence of tokens, $y_1,y_2,\dots,y_n\in  \R^{d}$, that is,
%\beq\label{map F Gamma}
%F_\Gamma:(x_1,x_2,\dots,x_n)\to (y_1,y_2,\dots,y_n),
%\eeq
%where $y_i=\Gamma(\mu,x_i)$.
%\subsection{Composition of MLP and attention and the operations for discrete measures}

We consider MLPs with possible skip connections, denoted by $F_\eta$, that are given by the function
\begin{align}\label{recall FLP}
    F_\eta:\R^d\to \R^d,\quad  F_\eta = c_\eta\cdot Id_x+\sigma\circ (A^L_\eta +b^L_\eta)\circ 
    \dots\circ  \sigma\circ (A^1_\eta +b^1_\eta),
\end{align}
where $c_\eta\in \R$, $A_\eta^j\in \R^{d_j\times d_{j-1}}$ are the weight matrices, $b_\eta^j\in \R^{d_{j}}$ are
 bias vectors, $\sigma$ is an activation function, for example the sigmoid
function, and $d_0=d_{L}=d$. This defines a map for measures, $f_{F_\eta}= (F_\eta)_\#:\mathcal M^+(\R^d)\to\mathcal M^+(\R^d)$ that for 
discrete measure $\nu=\sum_{i=1}^n\frac 1n \delta_{y_i}$ is given by
\begin{align}
f_{F_\eta}(\nu)=(F_\eta)_\#\bigg(\sum_{i=1}^n\frac 1n \delta_{y_i}\bigg)=\sum_{i=1}^n\frac 1n \delta_{z_i}
\end{align} 
where 
\begin{align}
  z_i=F_\eta(y_i).
\end{align}
The composition $f_{F_\eta} \diamond f_{\Gamma_\xi}:\mathcal M^+(\Omega)\to\mathcal M^+(\R^d)$ of the maps $f_{F_\eta}$ and $f_{\Gamma_\xi}$,
see \eqref{diamond composition}, maps the discrete measure  $\mu$,
given in \eqref{mu measure1}, to 
\begin{align}\label{composition with mu measure1}
(f_{F_\eta} \diamond f_{\Gamma_\xi})\bigg(\sum_{i=1}^n\frac 1n \delta_{x_i}
\bigg)=\sum_{i=1}^n\frac 1n \delta_{z_i},\quad
z_i=F_\eta(\Gamma_\xi(\mu,x_i)).
\end{align} 
We write
$$
H_\eta(x):=F_\eta(x)-x.
$$
We note that as $\Gamma_\xi(\mu,x)=x+\hbox{Att}_\xi(\mu,x)$ and
$F_\eta(x)=x+H_\eta(x)$,
we can write
\begin{align}\label{composition}
F_\eta(\Gamma_\xi(\mu,x))=
x+{\mathcal V}(\mu,x)
%=
%x+\hbox{Att}(\mu,x)+(H_\eta\circ (Id+ \hbox{Att}(\mu,\cdot)))(x)
\end{align}
and
\begin{align}\label{composition and map mathcal V}
& f_{F_\eta} \diamond f_{\Gamma_\xi}=Id_x+f_{\mathcal V},
%& {\mathcal V}:= f_V ,
\end{align}
where  
${\mathcal V}:
\mathcal M^+(\Omega)\times \R^d\to \R^d$
is the map
${\mathcal V}=\hbox{Att}_\xi+H_\eta\circ \Gamma_\xi$, that is,
\begin{align}\label{map V}
{\mathcal V}(\mu,x)=\hbox{Att}_\xi(\mu,x) + H_\eta(\Gamma_\xi(\mu,x))= \hbox{Att}_\xi(\mu,x)+H_\eta \circ (Id_x + \hbox{Att}_\xi(\mu,\cdot))(x).
\end{align} 
%\begin{align}\label{diamond composition}
%    (\mu,x) \mapsto (\Gamma_2 \diamond \Gamma_1)(\mu,x) := \Gamma_2( \nu, \Gamma_1(\mu,x)),\quad \nu := \Gamma_1(\mu)_\sharp \mu.
%\end{align}

\section{A counterexample for the characterization of support-preserving maps using only continuity 
%in 1-Wasserstein topology and the finite support property
}
\label{app:counterexample}

In this section, we construct a map $f:{\mathcal P}(\Omega)\to {\mathcal P}(\Omega)$ that is support preserving and continuous in the 1-Wasserstein topology, but which cannot be represented as $f_G$ using a continuous in-context map, $G$. Such a map, $f$, is given in formulas \eqref{Tb map}-\eqref{Tb map3} below. 
This shows the importance of the assumptions on the derivative of the map $f$ in the main theorem.

%As a function $G:\RR^d\to \RR^d$ is continuous if and only if for all Borel sets $A\subset \RR^d$  
%the pre-image $G^{-1}(A)$ is a Borel set, the continuity is crucial to consider the push forward operation
%$G_\#:\mu \to  G_\#\mu$ as a function that maps Borel-measures to Borel-measures.

Let us next prove Proposition \ref{prop-non existance}. We recall the statement:

\begin{proposition}\label{prop-non existance recall}
Let $d = 1$ and $\Omega = [-3,3] \subset \R$ and consider the set ${\mathcal P}(\Omega)$ endowed with the 1-Wasserstein topology. There exists a continuous, support preserving map $f : {\mathcal P}(\Omega) \to {\mathcal P}(\Omega)$ such that there does not exist a continuous map $G:{\mathcal P}(\Omega)\times \Omega\to \Omega$ for which $f=f_G$. 
%Such a map, $f$, is given in formulas \eqref{Tb map}-\eqref{Tb map3} below. 
%This shows the importance of the assumptions on the derivative of the map $f$ in the main theorem.
\end{proposition}

\begin{proof}

 For $0 \le a \le 1$, we define 
\beq
& &R_a:[-3,3]\to [-3,3],\\ \nonumber
& &R_a(x)=\begin{cases}\ \ \ x,& \hbox{for $x\le -1$ or $x\ge 1$,}\\
x+\frac 1{10}\cos^2(\frac 12 \pi x) \cos( ax),& \hbox{for $-1 < x< 1$.}
%\\
%-|x|^b & \hbox{for $-1< x< 0$}
\end{cases}
\eeq
We note that the derivative of $R_a$ is given by
\beq
& &R'_a:[-3,3]\to \R,\\ \nonumber
& &R'_a(x)=\begin{cases}\ \ \ 0,& \hbox{for $x\le -1$ or $x\ge 1$,}\\
1
-
\frac \pi{10}\cos(\frac 12 \pi x)\sin(\frac 12 \pi x) \cos( ax)
-
\frac a{10}\cos^2(\frac 12 \pi x) \sin( ax),
%
%-\frac {2\pi a} {10} \cos^2(\pi ax)\sin(\pi ax)
& \hbox{for $-1 < x< 1$}
%\\
%-|x|^b & \hbox{for $-1< x< 0$}
\end{cases}
\eeq
and that  $R_a:[-3,3]\to [-3,3]$ is a $C^1$ function that maps $R_a :\ [-3,3]\to [-3,3]$. Moreover, we point out that when $a=0$, $R_0(x)=x$.

Next, we consider the map
\beq\label{Tb map}
f(\mu)=(R_{a(\mu)})_\#\mu,
\eeq
where
\beq \label{Tb map2}
a(\mu)=
\begin{cases}\ \ \ \frac 1{\kappa(\mu)}& \hbox{if $\kappa(\mu)>0$,}\\
\ \ \ \ \ 0&  \hbox{if $\kappa(\mu)=0$}\\
%\\
%-|x|^b & \hbox{for $-1< x< 0$}
\end{cases}
\eeq
and 
\beq \label{Tb map3}
\kappa(\mu)
%=\mu([-1,1])%
%=\int_0^{1} xd\mu(x)+\int_{-1}^0(-x)d\mu(x)
%=\int_{-1}^{1} (1-|x|)^{1/2}d\mu(x),%+\int_{-1}^0(-x)d\mu(x)
=\int_{-2}^{-1} (2-|x|)d\mu(x)+\int_{-1}^{1} d\mu(x)+\int_{1}^{2} (2-x)d\mu(x).
%+\int_{-1}^0(-x)d\mu(x)
\eeq
The function $\mu \to \kappa(\mu)$ is a continuous map ${\mathcal P}([-3,3])\to \R$ but $\mu \to a(\mu)$ is not continuous.

%It seems that the map \eqref{Tb map} satisfies conditions (A1)-(A3). Let show this
%in the case when $p=1$. Then,
%see \url{https://en.wikipedia.org/wiki/Wasserstein_metric}
%S. S. Vallander, Theory of Probability and its Applications, 1974, Volume 18, Issue 4, Pages 784–786 ,
 By \cite{Vallender}, the $1$-Wasserstein distance satisfies,
 \beq
W_{1}(\mu _{1},\mu _{2})= \int _{ {[-3,3]} }\left|F_{1}(x)-F_{2}(x)\right|\,\mathrm{d} x,
\eeq
where $F_1(x)=\mu_1([-3,x])$ and $F_2(x)=\mu_2([-3,x])$ are the cumulative distribution functions of $\mu_1$ and $\mu_2$, respectively. Thus, as  $R_{a(\mu)}(x)$ is the identity map
for $x \in [-3,3] \setminus  [-\frac 32,\frac 32]$ and $R_{a(\kappa)}$ maps the interval
$ [-\frac 32,\frac 32]$ to itself, we find that
\beq\label{Wasserstein distance formula}
W_{1}(f(\mu _{1}),f(\mu _{2}))&\le &\hbox{diam}\bigg(\bigg[-\frac 32,\frac 32\bigg]\bigg) \bigg|\mu _{1}\bigg(\bigg[-\frac 32,\frac 32\bigg]\bigg)-\mu_{2}\bigg(\bigg[-\frac 32,\frac 32\bigg]\bigg)\bigg|+
W_{1}(\mu _{1},\mu _{2})
\nonumber\\
&\le &3\bigg(\mu _{1}\bigg(\bigg[-\frac 32,\frac 32\bigg]\bigg)+\mu_{2}\bigg(\bigg[-\frac 32,\frac 32\bigg]\bigg)\bigg)+
W_{1}(\mu _{1},\mu _{2}). 
%\\
%&\le & 4 \e+W_{1}(\mu _{1},\mu _{2}).
\eeq

\begin{lemma}
The map, $f :\ {\mathcal P}([-3,3]) \to {\mathcal P}([-3,3])$, is continuous in the 1-Wasserstein topology and is a support-preserving map.
\end{lemma}

\begin{proof}
When $\nu= \sum_{i=1}^{n}\frac 1n \delta_{x_i}$, we have by the definition of $f$ (cf.~\eqref{Tb map}) that
\ba
f(\nu)=(R_{a_0})_\#\nu,
\ea 
where $a_0=a(\nu)$. As $R_{a_0}$ is a $C^1$-map, we see that
\beq
f(\nu)= \sum_{i=1}^{n}\frac 1n \delta_{y_i},\quad y_i=R_{a_0}(x_i).
\eeq
This shows that $f$ is a  support-preserving map.

 Let $\mu_k,\mu \in {\mathcal P}([-3,3])$
satisfy 
\beq\label{without f limit in n}
\lim_{k\to\infty} W_{1}(\mu_k,\mu)=0.
\eeq
%$\mu_k\to \mu $ as $n\to \infty$, in the 1-Wasserstein metric.
We will next show that
\beq\label{limit in n}
\lim_{k\to\infty} W_{1}(f(\mu _k),f(\mu))=0.
\eeq

First, we consider the case when $\kappa(\mu)>0$. In this case, also $\kappa(\mu_k)>0$ when $n$ 
is large enough. Then, we can use the fact that $(x,a)\to R_a(x)$
is $C^1$-smooth in the domain $(x,a)\in [-3,3]\times (0,1]$, i.e., when
$a$ is strictly positive. This implies that  
 the limit \eqref{limit in n} is valid when
$\kappa(\mu)>0$.

Second, we consider the case when $\kappa(\mu)=0$. Then, 
$\mu([-\frac 32-\frac 1{10},\frac 32+\frac 1{10}])=0$ and $f(\mu)=\mu$.
For all $\e>0$
there is $n_\e>0$ such that for $n\ge n_\e$ it holds that
$W_{1}(\mu _n,\mu)<\e$ and
$\mu_k([-\frac 32,\frac 32])<\e$. % and $\kappa(\mu_k)<\e$. 

%If $\kappa(\mu_k)=0$, we have $f(\mu_k)=\mu_k$ and thus
%%\eqref{without f limit in n}
%%impl
%%and
%\beq
% W_{1}(f(\mu _n),f(\mu))<\e.
%\eeq
%
%If $\kappa(\mu_k)>0$, we 
We see that  $R_{a(\mu)}(x)$ is the identity map
for $x\in [-3,3]\setminus  [-\frac 32,\frac 32]$ and $R_{a(\kappa)}$ maps the interval
$ [-\frac 32,\frac 32]$ to itself.
Thus,  $n\ge n_\e$, we have by \eqref{Wasserstein distance formula},
\beq
W_{1}(f(\mu_k),f(\mu ))
&\le &3\bigg(\mu_k\bigg(\bigg[-\frac 32,\frac 32\bigg]\bigg)+\mu\bigg(\bigg[-\frac 32,\frac 32\bigg]\bigg)\bigg)+
W_{1}(\mu_k,\mu) 
\nonumber\\
&\le &3 \e+W_{1}(\mu_k,\mu)
\nonumber\\[0.25cm]
&\le &4 \e.
\eeq
These show that the limit \eqref{limit in n} is valid also when
$\kappa(\mu)>0$. This proves that the limit \eqref{limit in n} is valid.
Hence, $f$ is continuous in 1-Wassestein metric. This proves the claim.
\end{proof}

%, and hence
%we can use the fact that $(x,a)\to R_a(x)$
%is $C^1$-smooth in the domain $(x,a)\in [-10,10]\times (0,1]$, i.e., when
%$a$ is strictly positive. This implies that  when 
%\beq
%\lim_{n\to\infty} W_{1}(f(\mu _n),f(\mu))=0.
%\eeq

In the following, we use the 1-Wasserstein topology in the set $\mathcal P([-3,3])$.

\begin{lemma}
There are no continuous maps $G:\mathcal P([-3,3])\times [-3,3]\to [-3,3]$,
such that
\beq
f(\mu)=f_G(\mu).
\eeq
\end{lemma}

\begin{proof}
For $\e>0$, let
\ba
& &\mu_\e=(1-\epsilon)\delta_{x_0}+\epsilon \delta_{\sqrt \e},\\
& &\nu_\e=(1-\epsilon)\delta_{x_0}+\epsilon \delta_{R_{1/\e}(\sqrt \e)},
\ea
where $x_0=2$.
We see that  as $\epsilon\to 0$, we have
%that is, in $\mathcal P([-3,3])$,
\beq\label{lim of mu eps}
& &\lim_{\epsilon\to 0} W_1(\mu_\e,\delta_{x_0})=0,\\
& &\lim_{\epsilon\to 0} W_1(\nu_\e,\delta_{x_0})=0.
\eeq 
We have $\kappa(\mu_\e)=\e$ so that $a(\e)=1/\e$ and thus
we see that
\beq
f(\mu_\e)=\nu_\e.
\eeq
Moreover,
\beq
R_{a(\mu_\e)}(\sqrt \e)&=&R_{1/\e}(\sqrt \e)
\\
&=&\sqrt \epsilon +
\frac 1{10}\cos^2\bigg(\frac 12 \pi \sqrt \e\bigg) \cos\bigg( \frac 1{\e}\sqrt \e\big)\\
&=&
\sqrt \epsilon +\frac 1{10}\cos^2\bigg(\frac 12 \pi \sqrt \e\bigg) \cos\bigg(\frac 1{\sqrt \e}\bigg).
\eeq

Let us assume that there is a continuous map $G:\mathcal P([-3,3])\times [-3,3]\to [-3,3]$, where in the set $\mathcal P([-3,3])$ we use the 1-Wasserstein topology such that 
\beq
f(\mu)=f_G(\mu)=(G(\mu))_\#\mu.
\eeq
We observe that 
\beq
f(\mu_\e)=\nu_\e
\eeq
implies that when $0<\e<\frac 12$ we have
$\mu_\e\in \mathcal M^+_{fin,dif}([-3,3])$ and 
\beq
& &G(\nu_\e,x)|_{x=2}=2,\\
& &G(\nu_\e,x)|_{x=\sqrt \epsilon}={R_{1/\e}(\sqrt \epsilon)}.
\eeq
Thus,
\beq\label{limsup formula}
& &\limsup_{\e\to 0+%,\e\in \mathbb Q
}G(\mu_\e,x)\bigg|_{x=\sqrt \epsilon}=
\limsup_{\e\to 0+}\sqrt \epsilon +\frac 1{10}\cos^2\bigg(\frac 12 \pi \sqrt \e\bigg) \cos\bigg(\frac 1{\sqrt \e}\bigg)=+\frac 1{10},\\
\label{liminf formula}
& &\liminf_{\e\to 0+%,\e\in \mathbb Q
}G(\mu_\e,x)\bigg|_{x=\sqrt \epsilon}=
\liminf_{\e\to 0+}\sqrt \epsilon +\frac 1{10}\cos^2\bigg(\frac 12 \pi \sqrt \e\bigg) \cos\bigg(\frac 1{\sqrt \e}\bigg)=-\frac 1{10}.
\eeq
Formulas \eqref{limsup formula}, \eqref{liminf formula}, and \eqref{lim of mu eps} are in contradiction with the assumption that 
the map $G:\mathcal P([-3,3])\times [-3,3]\to [-3,3]$ is continuous.
This proves the claim.
\end{proof}

The above lemmas yield {Proposition} \ref{prop-non existance}.
\end{proof}

{

%AN UPDATED TEXT (Sep. 11):

To discuss the connection of the above counterexample with LLMs, we consider a sequence of tokens $(x_1,x_2,\dots,x_n)\in \Omega^n$, where $\Omega\subset \RR^d$, that are identified
with discrete measures $\frac 1n \sum_{i=1}^n\delta_{x_i}$ via the map $\iota$ given in formula \eqref{iota map}.
Below, as an interesting counterexample, we will construct a map $f:\mathcal M^+_{fin}([-3,3]) \to \mathcal M^+_{fin}([-3,3])$
for which the corresponding map $F=\iota^{-1}\circ f\circ \iota$ maps
a sequence of tokens $X=(x_1,x_2,\dots,x_n)\in [-3,3]^n\subset \R^n$ to a sequence
\ba
F(X)=(y_1(X,x_1),y_1(X,x_2),\dots,y_n(X,x_n))\in [-3,3]^n\subset \R^n.
\ea

Let us consider an example where $d=1$ and let $\Omega=[-3,3]$ be the space where we consider the tokes and $B_1=[-1,1]$ and $B_2=[-2,2]$ be balls (i.e. intervals) centered at zero.

This map has the following property: Let $n>1$ be very large and consider a sequence $X_n=(x_1,x_2,\dots,x_n)$
where
\ba
x_1,x_2,x_3\in B_1,\quad x_4,x_5,x_6,x_7,\dots,x_n\in\Omega\setminus B_2
\ea
that is, the first three tokens are in the smaller neighborhood of the point 0
and all other tokens are outside the larger neighborhood of the point 0. Denote the image of this sequence of tokens in the map $F$ by
\ba
F(X_n)=F(x_1,x_2,\dots,x_n)=(y_1(X_n,x_1),y_2(X,x_2),\dots,y_n(X,x_n)).
\ea
When $n$ is large, the measure $\mu_X(B_2)$, of the set $B_2$ with respect
to the measure $\mu_X=\frac 1n \sum_{i=1}^n\delta_{x_i}$, is small. More precisely, $\mu_X(B_2)=\frac 3n$.
Then, when $f$ is the map constructed below in formulas \eqref{Tb map}-\eqref{Tb map3} below, 
the function $$x_1\to (y_1(X_n,x_1),y_2(X_n,x_2),y_3(X_n,x_3))$$ converges to a discontinuous function as $n\to \infty$.
This means that when the prompt becomes sufficiently long,
then the map $F$ transforms some of the tokens in a possible unstable way.
As a possible playful example, two long, almost similar  prompts, coded with map which
assigns tokens in $\R^d$ for words, so that the names `Alice' and `Elise' are mapped to tokens that
are very close to each others, that is, $|\iota(Alice)-\iota(Elise)|$ is small.
We consider to very long prompts which are the same except their first words (in this example, we use
a long `Lorem ipsum' text that is commonly used in graphic design and publishing as a dummy or placeholder text).
The promts
%with a positional embedding of words
%that slightly changes the tokens of words according to the position where they appear in the prompt,
\ba
& &X_n=(Alice,is,studying,the,text,Lorem,ipsum,dolor, sit, amet, \dots, nibh),\\
%(Dog,nice,nice,nice,\dots,nice),\\
& &X'_n=(Elise,is,studying,the,text,Lorem,ipsum,dolor, sit, amet,\dots, nibh),
%(nice,Dog,nice,nice,\dots,nice)
\ea
could possibly be mapped in the composition of $F$ and a permutation $S$ (that changes the 1st and the 3rd words)
\ba
& &(S\circ F)(X_n)=(The,reader,LOVES,the,text,Lorem,ipsum,dolor, sit, amet, \dots, nibh),\\
%(Dog,nice,nice,nice,\dots,nice),\\
& &(S\circ F)(X'_n)=(The,reader,HATES,the,text,Lorem,ipsum,dolor, sit, amet, \dots, nibh).
%(nice,Cat,nice,nice,\dots,nice).
\ea
%As a small perturbation of the input the $f$ can chan
%Next, we consider the mathematical construction of the map $f$ having properties analogous to those of the above %map.
}

\end{document}